\documentclass{article}

\usepackage{graphicx}
\usepackage{amsmath, amssymb, color}
\usepackage{natbib}
\usepackage{bbm}
\usepackage{url}
\usepackage[lined,figure]{algorithm2e}

\newcommand{\BlackBox}{\rule{1.5ex}{1.5ex}}  
\newenvironment{proof}{\par\noindent{\bf Proof\ }}{\hfill\BlackBox\\[2mm]}
 
\newtheorem{theorem}{Theorem}
\newtheorem{lemma}[theorem]{Lemma} 
\newtheorem{remark}[theorem]{Remark}
\newtheorem{corollary}[theorem]{Corollary}

\newtheorem{problem}{Problem}

\begin{document}

\title{\Large\textsc{Sparse Learning over Infinite Subgraph Features}}

\author{
 Ichigaku Takigawa\footnote{Creative Research Institution, Hokkaido
 University. Email: takigawa@cris.hokudai.ac.jp.} 
 \hspace{1pt} and \hspace{1pt}
 Hiroshi Mamitsuka\footnote{Institute for Chemical Research, Kyoto
 University. Email: mami@kuicr.kyoto-u.ac.jp.}
}
\maketitle

\begin{abstract}
 We present a supervised-learning algorithm from graph data---a set of
 graphs---for arbitrary twice-differentiable loss functions and sparse
 linear models over all possible subgraph features. 
 To date, it has been shown that under all possible
 subgraph features, several types of sparse learning, such as
 Adaboost, LPBoost, LARS/LASSO, and sparse PLS regression, can be
 performed.
Particularly emphasis is placed on simultaneous learning of relevant
 features from an infinite set of candidates.
We first generalize techniques used in all these preceding studies to
 derive an unifying bounding technique for arbitrary separable functions. 
We then carefully use this bounding to make block
 coordinate gradient descent feasible over infinite subgraph features,
 resulting in a fast converging algorithm that can solve a wider class of
 sparse learning problems over graph data.
 We also empirically study the differences from the existing approaches in
 convergence property, selected subgraph features, and search-space
 sizes. We further discuss several unnoticed issues in sparse learning
 over all possible subgraph features.
\end{abstract}

\section{Introduction}


We consider the problem of modeling the response $y \in \mathcal{Y}$ to an input graph
$g \in \mathcal{G}$ as $y=\mu(g)$ with a model function $\mu$ from $n$ given observations
\begin{equation}\label{setting}
 \{(g_1,y_1),(g_2,y_2),\dots,(g_n,y_n)\}, \quad g_i \in
 \mathcal{G}, \quad y_i \in \mathcal{Y},
\end{equation}
where $\mathcal{G}$ is a set of all finite-size, connected, node-and-edge-labeled,
undirected graphs, and $\mathcal{Y}$ is a label space, i.e. a set
of real numbers $\mathbb{R}$ for regression and a set of nominal or
binary values such as $\{T,F\}$ or $\{-1,1\}$ for classification
\citep{Kudo:2005,Tsuda:2007,Saigo:2009,Saigo:2008a}.




Problem \eqref{setting} often arises in life sciences to understand
the relationship between the function and structure of biomolecules. A
typical example is the Quantitative Structure-Activity Relationship (QSAR)
analysis of lead chemical compounds in drug design, where the topology
of chemical structures is encoded as molecular
graphs~\citep{Takigawa:2013}. The properties of
pharmaceutical compounds in human body---typically of interest---such as
safety, efficacy, ADME, and
toxicity involve many contributing factors at various levels:
molecules, cells, cell population, tissues, organs and individuals
(entire body).
This complexity makes physico-chemical simulation hard, and hence,
statistical modeling using observed data becomes a more powerful
approach to quantify such complex properties.
Other examples of objects encodable as graphs include nucleotide or amino-acid
sequences \citep{Vert:2006}, sugar chains or glycans
\citep{Yamanishi:2007,Hashimoto:2008}, RNA secondary structures
\citep{Karklin:2005,Hamada:2006}, protein 3D structures
\citep{Borgwardt:2005}, and biological networks \citep{Vert:2007}. 
Graphs are pervasive data structures in not only life sciences but also
a variety of other fields, especially computer sciences, where typical
examples are images \citep{Harchaoui:2007,Nowozin:2007} and natural
language \citep{Kudo:2005}. 

\subsection{Problem Setting}

\begin{figure}[t]
\centering
\includegraphics[width=140mm]{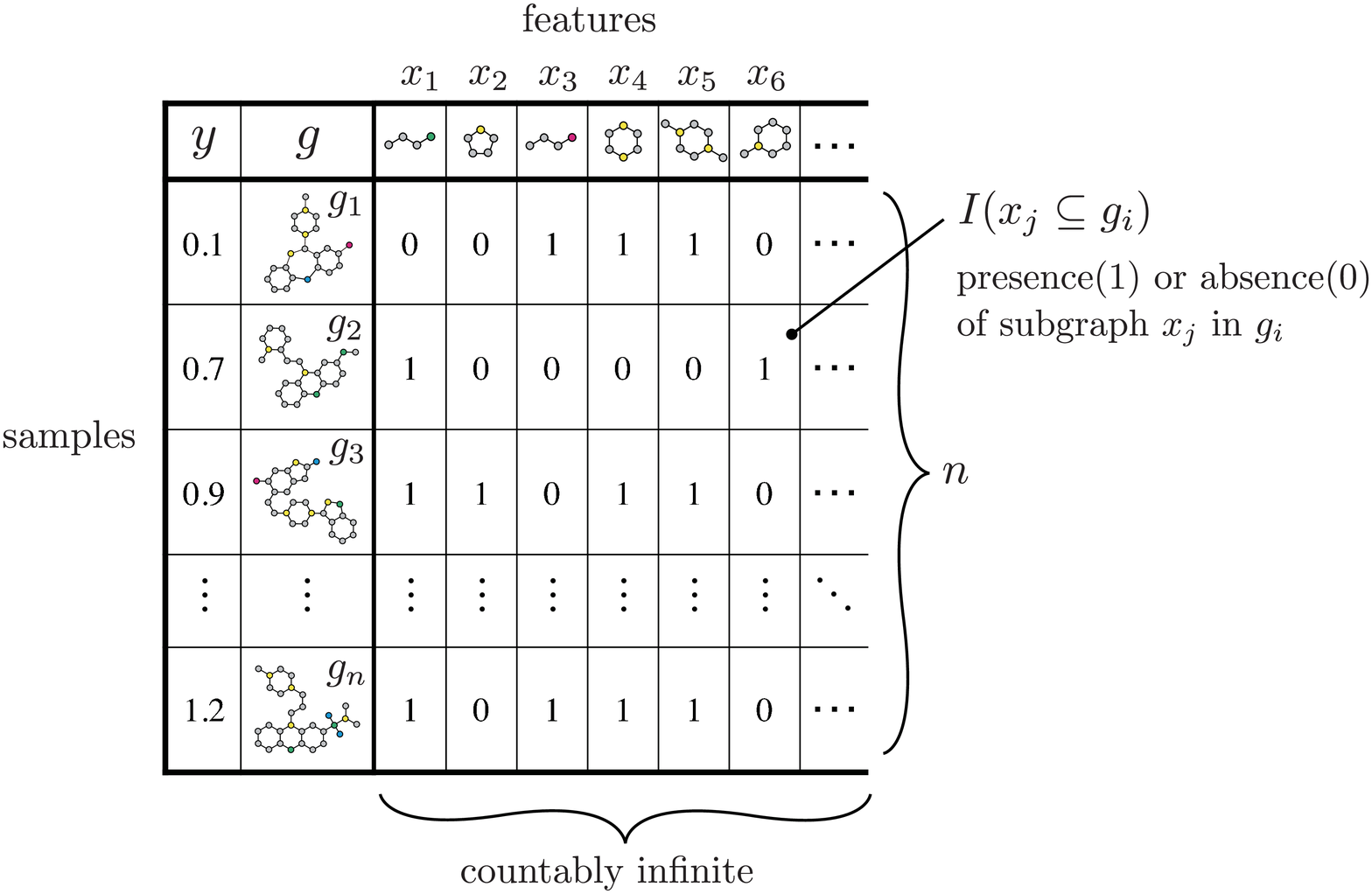}
\caption{The design matrix for $g_1,g_2,\dots,g_n$ defined by subgraph
 indicators. This matrix can be obtained explicitly if the size of
 subgraph features is limited beforehand. 
In general, however, typical subgraph features are countably infinite,
 resulting in that we cannot obtain this matrix explicitly.}
 \label{fig:feat_table}
\end{figure}

Our interest in Problem \eqref{setting} is to statistically model the
response $y$ to an input graph $g$ as $y=\mu(g)$ with a model function $\mu$.
Whether explicitly or not, existing methods use
\textit{subgraph indicators} as explanatory variables. A subgraph
indicator $I(x\subseteq g) \in \{0,1\}$ indicates the presence or
absence of a subgraph feature $x \in \mathcal{G}$ in the input graph $g
\in \mathcal{G}$.  

Once if we fix explanatory variables to use, this problem takes a typical
form of supervised-learning problems. We can then think of the
sample-by-feature design matrix 
that is characterized by the set of all possible subgraph indicators.
Figure \ref{fig:feat_table} shows an example of the design matrix.
However, we cannot explicitly obtain this matrix in practical
situations. The number of possible subgraphs is countably infinite, and
it is impossible in general to enumerate all possible
subgraphs beforehand. Moreover, we need to test subgraph isomorphism
to obtain each of $0$ or $1$ in the matrix. Only after we query a
particular subgraph $x_i \in \mathcal{G}$ to a given set of $n$ graphs, 
$g_1,g_2,\dots,g_n$, we can obtain the corresponding column vector to
$x_i$. At the beginning of the learning phase, we have no column vectors
at all.

Hence the main difficulty of statistical learning problems from a set of
graphs is \textit{simultaneous learning} of the model function $\mu$ and
a \textit{finite} set of relevant features $\{x_1, x_2, x_3, \dots\} \subset
\mathcal{G}$. In practice we need to select some features among all possible
subgraph features, but we do not know in advance subgraph features that
are most relevant to a given problem. We thus need to search all
possible subgraph features \textit{during} the learning process. 

In this paper, we consider a standard linear model over all possible subgraph
indicators 
\[
 \mu(g; \beta,\beta_0) := \beta_0 + \sum_{j=1}^\infty \beta_j
  I(x_j \subseteq g), \quad \beta = (\beta_1,\beta_2,\dots),
\]
where we assume the \textit{sparsity} on coefficient parameters: Most
coefficients $\beta_1,\beta_2,\dots$ are zero and a few of
them, at least a finite number of them, are nonzero. We write $\mu(g;
\beta,\beta_0)$ simply as a function of $g$, $\mu(g)$, whenever there is
no possibility of confusion.

\subsection{Contributions of This Work}

\begin{itemize}
 \item \textbf{Main Result:} We present a generic algorithm that
       solves the following class of supervised-learning
       problems 
\begin{equation}\label{maintarg}
  \min_{\beta,\beta_0} \sum_{i=1}^n
L\bigl(y_i,\mu(g_i;\beta,\beta_0)\bigr) + \lambda_1 \|\beta\|_1 + \frac{\lambda_2}{2}
\|\beta\|_2^2, 
\end{equation}
where $L$ is a twice differentiable loss function and
       $\lambda_1,\lambda_2 \geqslant 0$. We here pose the elastic-net
       type regularizer: The third term of 2-norm $\|\beta\|_2^2$
       encourages highly correlated
       features to be averaged, while the second term of 1-norm $\|\beta\|_1$, a
       sparsity-inducing regularizer, encourages a sparse solution in
       the coefficients of these averaged features. Subgraph features
       have a highly correlated nature, due to the inclusion relation
       that $x_i$ can be a subgraph of another feature $x_j$, and due to
       the tendency that structural similarity between $x_i$ and $x_j$
       implies strong correlation between their indicators
       $I(x_i\subseteq g)$ and $I(x_j\subseteq g)$. Our framework
       targets a wider class of sparse learning problems than the
       preceding studies such as
       Adaboost \citep{Kudo:2005}, LARS/LASSO \citep{Tsuda:2007}, sparse
       partial-least-squares (PLS) regression
\citep{Saigo:2008a}, sparse principal component analysis (PCA)
\citep{Saigo:2008b}, and LPBoost \citep{Saigo:2009}.
\item The proposed algorithm directly optimizes the objective function
      by making \textbf{Block
      Coordinate Gradient Descent} 
      \citep{Tseng:2009,Yun:2011} feasible over all possible subgraph
      indicators. It updates
      \textit{multiple} nonzero features at each iteration with
      simultaneous subgraph feature learning, while existing methods are
      iterative procedures to find the best \textit{single} feature at each
      iteration that improves the current model most by a
      branch-and-bound technique. The
      proposed algorithm shows faster and more stable convergence.
\item We formally generalize the underlying bounding techniques behind
      existing methods to what we call \textbf{Morishita-Kudo
      Bounds} by extending the technique used for Adaboost
      \citep{Morishita:2002,Kudo:2005} to that for arbitrary \textit{separable}
      objective functions. This bounding is not only one of the keys to
      develop the proposed algorithm, but also gives a coherent and unifying
      view to understand how the existing methods work. We see this in
      several examples where this bounding easily and consistently gives
      the pruning bounds of existing methods. 
\item We also present another important technique, which we call
      \textbf{Depth-First Dictionary Passing}. It allows to establish the
      efficient search-space pruning of our proposed algorithm. This
       technique is important, because unlike
      existing methods, search-space pruning by the Morishita-Kudo bounds is not
      enough to making block coordinate gradient descent feasible over
      infinite subgraph features. 
\item We studied the convergence behavior of our method and
      the difference from two major existing approaches
      \citep{Kudo:2005,Saigo:2009} through numerical
      experiments. We used a typical example of 1-norm-penalized
      logistic regression estimated by the proposed algorithm. For
      systematic evaluations, we developed a random graph generator to
      generate a controllable set of graphs for binary
      classification. As a result,
      in addition to understanding the difference in the convergence
      behaviors, we pointed out and discussed several unnoticed issues including
      the size bias of subgraph features and the equivalence class of
      subgraph indicators.
\end{itemize}

\subsection{Related Work}\label{relatedwork}

Under the setting of Figure \ref{fig:feat_table}, the most well-used
machine learning approach for solving Problem \eqref{setting} would be
\textit{graph kernel method}. 
So far various types of graph kernels have been developed and also
achieved success in applications such as bioinformatics and
chemoinformatics \citep{Kashima:2003,Gartner:2003,Ralaivola:2005,
Frohlich:2006, Mahe:2005, Mahe:2006, Mahe:2009, Kondor:2008,
Kondor:2009, Vishwanathan:2010,Shrvashidze:2011}. For example, we can
compute the inner product of feature vectors of $g_i$ and $g_j$ in
Figure \ref{fig:feat_table}, indirectly through a kernel trick as
\[
k(g_i,g_j):=\text{\# of common subgraphs between $g_i$ and $g_j$},
\]
which is, as well as many other proper graph kernels, a special case
of \textit{R-convolution kernels} by \cite{Haussler:1999}.
However \cite{Gartner:2003} showed that (i) computing any complete graph
kernel is at least as hard as deciding whether two graphs are
isomorphic, and (ii) computing the inner product in the subgraph feature
space is NP-hard. Hence, any practical graph kernel
restricts the subgraph features to some limited types such as paths and
trees, bounded-size subgraphs, or
heuristically inspired subgraph features in application-specific
situations. 

In chemoinformatics, there have been many methods for computing
the fingerprint descriptor (feature vector) of a given
molecular graph. The predefined set of subgraph features engineered with
the domain knowledge would be still often used, such as the existence of
aromatic-ring or a particular type of functional groups. 
The success of this approach highly depends on whether
the choice of features well fits the intrinsic property of data. 
To address this lack of
generality, chemoinformatics community has developed data-driven
fingerprints, for example, hashed fingerprints (such as
Daylight\footnote{Daylight 
Theory Manual. Chapter 6: Fingerprints---Screening and
Similarity. \url{http://www.daylight.com/dayhtml/doc/theory/}} and
ChemAxon\footnote{Chemical Hashed Fingerprints.
\url{https://www.chemaxon.com/jchem/doc/user/fingerprint.html}}),
extended connectivity fingerprints (ECFP) \citep{Rogers:2010}, frequent
subgraphs \citep{Dashpande:2005}, and bounded-size graph fingerprint
\citep{Wale:2008}. The data-driven
fingerprints adaptively choose the subgraph features from a limited
type of subgraphs, resulting in that these fingerprints are similar to
practical graph kernels.

These two existing, popular approaches are based on subgraph
features, which are limited to specific type ones only.
In contrast, a series of inspiring studies for \textit{simultaneous
feature learning via sparse modeling} has been made
\citep{Kudo:2005,Tsuda:2006,Tsuda:2007,Tsuda:2008,Saigo:2008a,Saigo:2008b,Saigo:2009}.
These approaches involve automatic selection of relevant features from
\textit{all possible subgraphs} during the learning process.
Triggered by the seminal paper by \cite{Kudo:2005}, it has been shown
that we can perform simultaneous learning of not only the model
parameters but also relevant
subgraph features from all possible subgraphs in several
machine-learning problems such as Adaboost
\citep{Kudo:2005}, LARS/LASSO \citep{Tsuda:2007}, 
sparse partial-least-squares (PLS) regression
\citep{Saigo:2008a}, sparse principal component analysis (PCA)
\citep{Saigo:2008b}, and LPBoost \citep{Saigo:2009}. This paper aims to
give a coherent and unifying view to understand how these existing
methods work well, and then present a more general framework, which is
applicable to a wider class of learning problems. 

\section{Preliminaries}

In this section, we formally define our problem and notations~(Section
\ref{notation}), and summarize basic known results needed for the
subsequent sections~(Section \ref{freqmine}).

\subsection{Problem Formulation and Notations}\label{notation}

First, we formulate our problem described briefly in Introduction. 
We follow the 
traditional setting of learning from graph data.
Let $\mathcal{G}$ be a set of all finite-size, node-and-edge-labeled,
connected, undirected graphs with finite discrete label sets for nodes
and labels. We assume that data graphs $g_1,g_2,\dots,g_n$ and their subgraph
features $x_1,x_2,\dots$ are all in $\mathcal{G}$. We write the given
 set of $n$ graphs as
\[
 \mathcal{G}_n := \{g_1,g_2,\dots,g_n\},
\]
and the label space as $\mathcal{Y}$, for example,
$\mathcal{Y}=\mathbb{R}$ for regression and 
$\mathcal{Y}=\{T,F\}$ or $\{-1,1\}$ for binary classification. 

Then we study the following problem where the 1-norm regularizer
$\|\beta\|_1$ induces a sparse solution for $\beta$, i.e. a solution, in
which only a few of $\beta_1,\beta_2,\dots$ are nonzero. 
The nonzero coefficients bring the effect of \textit{automatic
selection} of relevant features among all possible subgraphs.
\begin{problem}\label{prob1}
For a given pair of observed graphs and their responses
\[
  \{(g_1,y_1),(g_2,y_2),\dots,(g_n,y_n)\}, \quad g_i \in
 \mathcal{G}, \quad y_i \in \mathcal{Y},
\]
solve the following supervised-learning problem:
\begin{equation}\label{learningproblem}
  \min_{\beta,\beta_0} \sum_{i=1}^n
L\bigl(y_i,\mu(g_i;\beta,\beta_0)\bigr) + \lambda_1 \|\beta\|_1 + \frac{\lambda_2}{2}
\|\beta\|_2^2, 
\end{equation}
where $L$ is a twice-differentiable loss function, $\lambda_1,\lambda_2
 \geqslant 0$, and the model function $\mu:
 \mathcal{G}\to\mathcal{Y}$ is defined as 
\begin{equation}\label{model}
 \mu(g; \beta,\beta_0) := \beta_0 + \sum_{j=1}^\infty \beta_j
  I(x_j \subseteq g), \quad \beta = (\beta_1,\beta_2,\dots),
\end{equation}
where $x_1,x_2,\dots \in \mathcal{G} $.
\end{problem}

Next, we define several notations that we use throughout the paper.
$I(A)$ is a binary indicator function of an event $A$,
meaning that $I(A)=1$ if $A$ is true; otherwise $I(A)=0$. The notation
$x\subseteq g$ denotes the subgraph isomorphism that $g$ contains a
subgraph that is isomorphic to $x$. Hence the subgraph indicator
$I(x\subseteq g) = 1$ if $x\subseteq g$; otherwise $0$.

Given $\mathcal{G}_n$, we define the union of all subgraphs of
$g\in\mathcal{G}_n$ as
\[
 \mathcal{X}(\mathcal{G}_n) := \{ x \in \mathcal{G}\mid x \subseteq g, g
 \in \mathcal{G}_n \}.
\]
It is important to note that $\mathcal{X}(\mathcal{G}_n)$ is a finite
set and is equal to the set of subgraphs each
of which is contained in at least one of the given graphs in $\mathcal{G}_n$ as
\[
 \mathcal{X}(\mathcal{G}_n) = \{ x \in \mathcal{G}\mid 
\text{$\exists g\in\mathcal{G}_n$ such that $x \subseteq g$} \}.
\] 
For given $\mathcal{X}(\mathcal{G}_n)$, we can construct an
\textit{enumeration tree} over $
\mathcal{X}(\mathcal{G}_n)$ that is fully defined and explained 
in Section \ref{freqmine}. Also we write a subtree of 
$\mathcal{T}(\mathcal{G}_n)$ rooted at $x \in
\mathcal{X}(\mathcal{G}_n)$ as $\mathcal{T}(x)$ (See Figure
\ref{fig:boolean}).

For given $\mathcal{G}_n$ and a subgraph feature $x$, we define the
\textit{observed indicator vector} of $x$ over $\mathcal{G}_n$ as 
\begin{equation}\label{regvec}
I_{\mathcal{G}_n}(x) :=
 (I(x\subseteq g_1),I(x\subseteq g_2),\dots,I(x\subseteq g_n)).
\end{equation}
From the definition, $I_{\mathcal{G}_n}(x)$ is an $n$-dimensional
Boolean vector, that is, $I_{\mathcal{G}_n}(x) \in \{0,1\}^n$, which
corresponds to each column vector of the design matrix in Figure \ref{fig:feat_table}.

For an $n$-dimensional Boolean vector $u := (u_1,u_2,\dots,u_n)\in\{0,1\}^n$, we write the
index set of nonzero elements and that of zero elements as
\[
1(u) := \{ i \mid u_i =1\} \subseteq \{1,2,\dots,n\},\qquad
0(u) := \{ i \mid u_i =0\} \subseteq \{1,2,\dots,n\}.
\]
From the definition, we have $1(u)\cup 0(u) = \{1,2,\dots,n\}$ and
$1(u)\cap 0(u) = \varnothing$. For simplicity, we also use the same notation
for the observed indicator vector of $x$ as
\begin{align*}
 1(x)&:=1(I_{\mathcal{G}_n}(x)) = \{i\mid x \subseteq g_i, g_i \in \mathcal{G}_n\},\\
 0(x)&:= 0(I_{\mathcal{G}_n}(x)) =\{i \mid x \not\subseteq g_i, g_i \in \mathcal{G}_n\}.
\end{align*}
whenever there is no possibility of confusion.

\subsection{Structuring the Search Space}\label{freqmine}

Our target model of \eqref{model} includes an infinite number of terms that
directly come from the infinite number of possible subgraph features $x_1,x_2,\dots$.
We first briefly describe the algorithmic trick of reducing a set of countably
infinite subgraph features to a well-structured \textit{finite} search
space which
 forms the common backbone of many preceding studies
\citep{Kudo:2005,Tsuda:2007,Saigo:2008a,Saigo:2009,Takigawa:2011a}.

In Problem \eqref{learningproblem}, we observe that the
model function $\mu(g)$ appears only as the function values at given
graphs, $\mu(g_1), \mu(g_2), \dots, \mu(g_n)$. Correspondingly, the term
$I(x_j \subseteq g)$ in the summation of the model \eqref{model} also only
appears as the values at $\mathcal{G}_n$, that is, as 
$I(x_j \subseteq g_1),I(x_j \subseteq g_2),\dots,I(x_j \subseteq g_n)$ for each $x_j$.
Hence, as far as Problem \eqref{learningproblem} concerns, 
we can ignore all subgraphs $x'$ that never occur in $g \in
\mathcal{G}_n$ because $I(x' \subseteq g_i)=0$ for
$i=1,2,\dots,n$ and they do not contribute to any final value of
$\mu(g_1),\mu(g_2),\dots,\mu(g_n)$. Therefore, for any $g \in \mathcal{G}_n$, we have
\[
 \mu(g; \beta,\beta_0) := \beta_0 + \sum_{j=1}^\infty \beta_j
  I(x_j \subseteq g) = \beta_0 + \sum_{x_j \in
  \mathcal{X}(\mathcal{G}_n)}
\beta_j
  I(x_j \subseteq g).
\]
Furthermore, it will be noticed that $\mathcal{X}(\mathcal{G}_n)$ equals
to a set of frequent subgraphs in $\mathcal{G}_n$ whose frequency
$\geqslant 1$, and the points so far can be summarized as follows.

\begin{lemma}\label{finiteness}
Without loss of generality, when we consider Problem \eqref{learningproblem},
we can limit the search space of subgraphs
 $x_j$ to the finite set
 $\mathcal{X}(\mathcal{G}_n)$ that is equivalent to a set of all
 frequent subgraphs in $\mathcal{G}_n$ with the minimum support
 threshold of one.
\end{lemma}

This fact connects Problem \eqref{learningproblem} to the
problem of enumerating all frequent subgraphs in the given set of
graphs---\textit{frequent subgraph pattern mining}---that has been
extensively studied in the data mining field. At this point,
since $\mathcal{X}(\mathcal{G}_n)$ is finite, 
we can say that Problem \eqref{learningproblem} with the linear
model \eqref{model} is solvable in theory if we have infinite time and
memory space. However, the set of $\mathcal{X}(\mathcal{G}_n)$ is
still intractably huge in general, and it is still practically impossible to
enumerate all subgraphs in $\mathcal{X}(\mathcal{G}_n)$ beforehand.

However the research on mining algorithms of frequent subgraph patterns 
brings a well-structured search space for $\mathcal{X}(\mathcal{G}_n)$,
called \textit{an enumeration tree}, which has two nice properties, to be
explained below.
The enumeration tree can be generated from a so-called Hasse diagram,
which is a rooted graph, each node being labeled by a subgraph, where
parent-child relationships over labels correspond to isomorphic
relationships and each layer has subgraphs with the same size.
Figure \ref{fig:enumtree} is schematic pictures of (a) a set of
subgraphs, $\mathcal{X}(\mathcal{G}_n)$, (b) a Hasse diagram over 
$\mathcal{X}(\mathcal{G}_n)$, and (c) an enumeration tree over
$\mathcal{X}(\mathcal{G}_n)$.
\begin{figure}[t]
\centering
\includegraphics[width=150mm]{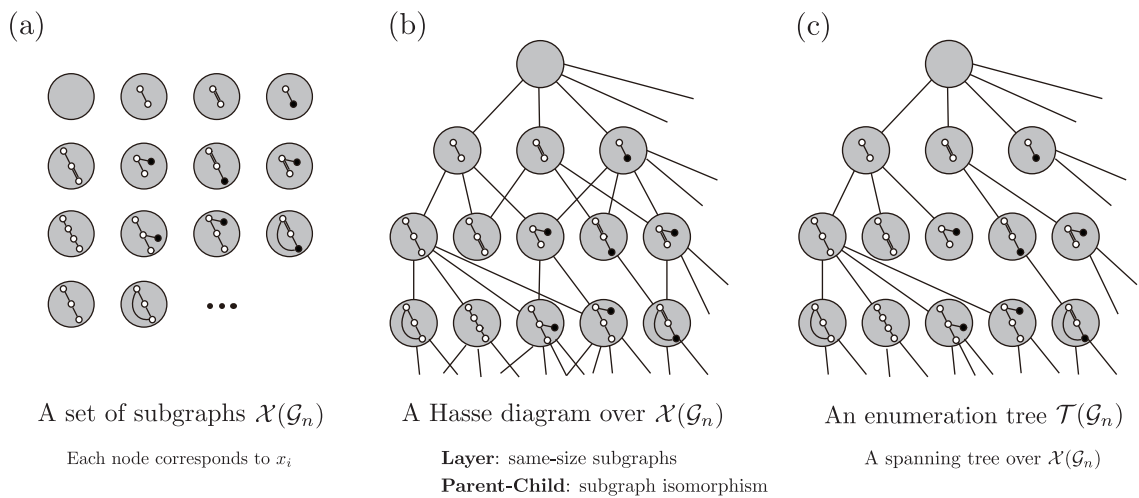}
\caption{Non-redundant check of all elements in
 set $\mathcal{X}(\mathcal{G}_n)$. (a) set $\mathcal{X}(\mathcal{G}_n)$,
 (b) Hasse diagram for the set $\mathcal{X}(\mathcal{G}_n)$, and (c)
a tree-shaped search space, enumeration tree
 $\mathcal{T}(\mathcal{G}_n)$ for the set 
 $\mathcal{X}(\mathcal{G}_n)$ which allows to search over a finite set
 of relevant subgraphs, $\mathcal{X}(\mathcal{G}_n)$, without any redundancy.
} \label{fig:enumtree}
\end{figure}

\begin{enumerate}
 \item Isomorphic parent-child relationship: Smaller subgraphs are
       assigned to the shallower levels, larger subgraphs to the deeper
       levels. The edge from $x_i$ to $x_j$ implies that $x_i$ and $x_j$
       are different by only one edge, and $x_i$ at the shallower level
       is isomorphic to the subgraph of $x_i$ at the deeper level. 
       This property is inherited from the Hasse diagram.
 \item Spanning tree: Traversal over the entire
       enumeration tree gives us a set of all subgraphs $x_j$ in
       $\mathcal{X}(\mathcal{G}_n)$, avoiding any redundancy in checking
       subgraphs, meaning that the same subgraph is not checked more
       than once.
\end{enumerate}

\noindent For example, the gSpan algorithm, a well-established algorithm for
frequent subgraph pattern mining, uses an enumeration tree by an
algorithmic technique called \textit{right-most extension}, where the
enumeration tree is traversed in a depth-first manner
\citep{Yan:2002,Takigawa:2011a}.  

Our main interest is to efficiently search the subgraphs that
should be checked, using enumeration tree
$\mathcal{T}(\mathcal{G}_n)$ to solve Problem 
\eqref{learningproblem} in terms of Model function
\eqref{model}. Throughout this paper, we use the
enumeration tree by the gSpan algorithm as a search space for 
$\mathcal{X}(\mathcal{G}_n)$.

The above fact can be formally summarized as follows.

\begin{lemma}\label{enumprop1}
Let $G:=(V,E)$ be a graph with a node set
 $V=\mathcal{X}(\mathcal{G}_n)\cap\{\varnothing\}$ and an edge set
$E=\{(x,x')\mid x\subseteq x',x \in V, x' \in
 V, \text{$x$ and $x'$ are different by only one edge}\}$, where
 $\varnothing$ denotes the empty graph. Then, this graph $G$ corresponds
 to the Hasse diagram in Figure \ref{fig:enumtree} (b). Moreover
 we can construct a spanning tree $\mathcal{T}(\mathcal{G}_n)$ rooted at
 $\varnothing$ over $G$, that is, an \textit{enumeration tree} for
 $\mathcal{X}(\mathcal{G}_n)$ that has the following properties.
\begin{enumerate}
 \item It covers all of $x \in \mathcal{X}(\mathcal{G}_n)$, where any of
       $x \in \mathcal{X}(\mathcal{G}_n)$ is reachable from the root
       $\varnothing$.
 \item For a subtree $\mathcal{T}(x)$ rooted at node $x$, we have $x
       \subseteq x'$ for any $x' \in \mathcal{T}(x)$.
\end{enumerate}
\end{lemma}

It is important to note that, from the second property in Lemma
\ref{enumprop1}, we have \begin{lemma}\label{enumprop2}
\[
 I(x \subseteq g) = 0 \Longrightarrow  I(x' \subseteq g) = 0, x' \in \mathcal{T}(x).
\]
\end{lemma}
In other words, once if we know $x \not\subseteq g$, then we can skip
checking of all $x' \in \mathcal{T}(x)$ because $x' \not\subseteq g$.
This fact serves as the foundation for many efficient algorithms of frequent
subgraph pattern mining, including the gSpan algorithm. 

\section{Morishita-Kudo Bounds for Separable Functions}

In this section, we formally generalize the underlying bounding
technique, which was originally developed for using Adaboost
      \citep{Morishita:2002,Kudo:2005} for subgraph features, to the
      one that we call \textbf{Morishita-Kudo Bounds}. 
Our generalization is an extension to that for arbitrary
\textit{separable} objective functions.
This bounding is not only one of the keys to
      develop the proposed algorithm but also gives a coherent and unifying
      view to understand how the existing methods work. 
In Section~\ref{mkb_sec}, we show several examples, in which this
bounding easily and consistently gives the pruning bounds of existing
methods. 

\subsection{Bounding for Branch and Bound}

Consider the problem of finding the best subgraph feature $x$ by using
given training data $(y_1,g_1)$, $(y_2,g_2), \dots, (y_n,g_n)$, when we
predict the response $y$ by \textit{only one} subgraph indicator $I(x
\subseteq g)$. This subproblem constitutes the basis for the existing
methods such as Adaboost \citep{Kudo:2005} and LPBoost \citep{Saigo:2009}.
The evaluation criteria can be various as follows, which
characterize individual methods:
\begin{enumerate}
 \item Maximizer of the weighted gain 
\[
 \textstyle\sum_{i=1}^n w_i y_i (2 I(x\subseteq g_i)-1)\,\,\text{for $y_i\in\{-1,1\}$}
\]
 \item Minimizer of the weighted classification error
\[
 \textstyle\sum_{i=1}^n w_i I( I(x\subseteq g_i) \neq y_i)\,\,\text{for $y_i\in\{0,1\}$}
\]
 \item Maximizer of the correlation to the response 
\[
 \textstyle{\Big|}\sum_{i=1}^n
       y_i I(x\subseteq g_i) {\Big|}\,\,\text{for $y_i\in\mathbb{R}$}
\]
\end{enumerate}
It should be noted that there exist multiple optimal subgraphs, and
any one of such subgraphs can be the best subgraph feature.

All of these problems take a form of finding the \textit{single}
subgraph $x^*$ that minimizes or maximizes some function $f(x)$
dependent on the given graphs $\mathcal{G}_n$. 
A brute-force way to obtain the best solution $x^* = \arg\min_x f(x)$
would be to first traverse all $x_j$ in
the enumeration tree $\mathcal{T}(\mathcal{G}_n)$, compute the value of
$f(x_j)$ at each node $x_j$, and then take the best one among
$\{f(x_j)\mid x_j \in \mathcal{T}(\mathcal{G}_n)\}$
after all $x_j$
are checked. Unfortunately this brute-force enumeration does not work in
reality because the size of the set $\mathcal{X}(\mathcal{G}_n)$ is
intractably large in most practical situations.

In existing work, this problem has been addressed by using \textit{branch
and bound}. Due to the nice property of an enumeration tree shown in Lemma
\ref{enumprop1} and \ref{enumprop2}, 
we can often know the upper and lower bounds of $f(x')$ for any
subgraph $x'$ in the subtree below $x$, i.e. $x'\in\mathcal{T}(x)$,
\textit{without} checking the values of every $f(x')$.
In other words, when we are at node $x$ of an enumeration tree during
traversal over the tree, we can obtain two values $\overline{c}(x)$ and
$\underline{c}(x)$ dependent on $x$ that shows
\[
 \underline{c}(x) \leqslant f(x') \leqslant
 \overline{c}(x)\,\,\text{for any $x' \in \mathcal{T}(x)$}
\]
\textit{without} actually checking $f(x')$ for $x' \in \mathcal{T}(x)$.

Then, if the tentative best solution $\hat{x}$ among already visited nodes
before the current $x$ is better than the best possible value
in the subtree $\mathcal{T}(x)$, i.e. $f(\hat{x}) < \underline{c}(x)$
(for finding the minimum $x^*$) or $f(\hat{x}) > \overline{c}(x)$ (for
finding the maximum $x^*$), we can prune the subtree $\mathcal{T}(x)$
and skip all checking of unseen subgraphs $x' \in \mathcal{T}(x)$.

In order to perform this branch and bound strategy, we need a systematic way to
get the upper and lower bounds for a given objective function. In 
existing work, these bounds have been derived separately for each specific
objective function such as Adaboost \citep{Kudo:2005}, LPBoost
\citep{Saigo:2009}, LARS/LASSO \citep{Tsuda:2007} and
PLS regression \citep{Saigo:2008a}. In the following subsections, we show that 
we can generalize the common technique behind these existing approaches
to what we call \textit{Morishita-Kudo Bounds}, which was developed for the
preceding studies based on Adaboost \citep{Morishita:2002,Kudo:2005}.

\begin{figure}[t]
\centering
\includegraphics[width=100mm]{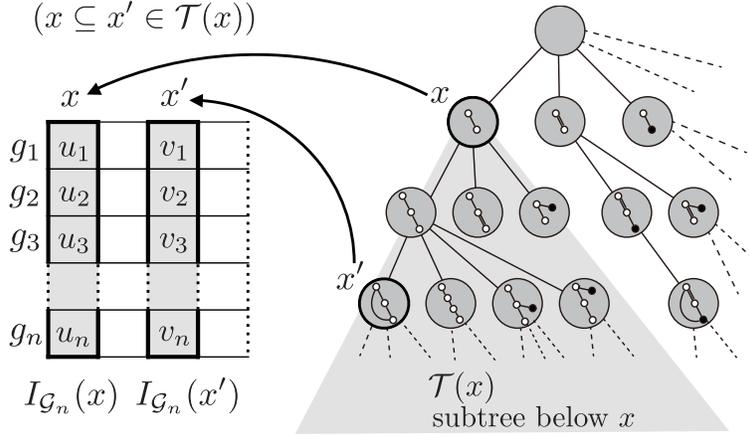}
\caption{Boolean vectors $I_{\mathcal{G}_n}(x)$ associated with $x \in
 \mathcal{T}(\mathcal{G}_n)$. In this example, $x \subseteq x'$. Hence
 $v_i=0$ for $u_i=0$. Only coordinates $v_i$ for $u_i=1$ can be either
 $0$ or $1$, as stated in Lemma \ref{boolprop}. 
} \label{fig:boolean}
\end{figure}

\subsection{Property of Boolean Vectors Associated with $\mathcal{G}_n$}

First, we observe that the following result holds from Lemma
\ref{enumprop1} and \ref{enumprop2}.

\begin{lemma}\label{boolprop}
$1(I_{\mathcal{G}_n}(x'))\subseteq 1(I_{\mathcal{G}_n}(x))$ for $x' \in
 \mathcal{T}(x)$. For short, $1(x') \subseteq 1(x)$ for $x' \in
 \mathcal{T}(x)$.
\end{lemma}

\begin{remark}
 Recall that $I_{\mathcal{G}_n}(x)$, previously defined as
 \eqref{regvec},
is the \textit{observed indicator vector}, an $n$-dimensional Boolean
 vector, indicating if $x$ is
 contained in each  of $\mathcal{G}_n$. The
 number of $1$s in the vector $I_{\mathcal{G}_n}(x)$, that is
 $|1(I_{\mathcal{G}_n}(x))|$, is identical to
 the ``support'' of $x$ in $\mathcal{G}_n$ in the standard data
 mining terminology.
\end{remark}

Lemma \ref{boolprop}, as well as Lemma \ref{enumprop2}, claims that when
we traverse an enumeration tree
down to the deeper level from $x$ to $x'$, the elements taking $1$ in
$I_{\mathcal{G}_n}(x)$ can change to $1$ or $0$ but the elements taking
$0$ must remain as $0$ in $I_{\mathcal{G}_n}(x')$. 
Figure \ref{fig:boolean} schematically shows this property of
enumeration trees.
More explicitly, let $I_{\mathcal{G}_n}(x)=(u_1,u_2,\dots,u_n)$
and $I_{\mathcal{G}_n}(x')=(v_1,v_2,\dots,v_n)$. Then,
\[
 v_i = \begin{cases}
	0 & \text{$(u_i = 0)$}\\
	\text{$0$ or $1$} & \text{$(u_i = 1)$}.
       \end{cases}
\]
This fact also implies that
the number of $1$s in the associated Boolean
vector $I_{\mathcal{G}_n}(x)$ at node $x$ monotonically decreases if we
proceed to any deeper node $x' (\supseteq x)$ in the enumeration tree. The \textit{anti-monotone
property} of the support, which is a fundamental technique in
frequent pattern mining, can also be obtained as a corollary of Lemma \ref{boolprop} as 
\[
x \subseteq x' \Longrightarrow |1(I_{\mathcal{G}_n}(x))| \geqslant |1(I_{\mathcal{G}_n}(x'))|.
\]

We also observe the following simple facts for arbitrary bounded
real-valued function on $n$-dimensional Boolean vector space,
$f:\{0,1\}^n \to \mathbb{R}$.

\begin{theorem}\label{maxmin}
Assume that $u=(u_1,u_2,\dots,u_n)\in\{0,1\}^n$ are given. Then, for any
 $v=(v_1,v_2,\dots,v_n)\in\{0,1\}^n$ such that $1(v)\subseteq 1(u)$, we
 have
\[
 \overline{f}(u) \geqslant f(v) \geqslant \underline{f}(u),
\]
where
\begin{align*}
\overline{f}(u)&=\max_{\alpha_i \in\{0,1\}, i\in 1(u)} \{f(\alpha_1,\alpha_2,\dots,\alpha_n) \mid \alpha_j=0, j \in 0(u)\},\\
\underline{f}(u)&=\min_{\alpha_i\in\{0,1\}, i\in 1(u)} \{f(\alpha_1,\alpha_2,\dots,\alpha_n) \mid \alpha_j=0, j \in 0(u)\}. 
\end{align*}
These $\overline{f}(u)$ and
 $\underline{f}(u)$ are non-trivial bounds of $f(v)$ when they respectively satisfy
\[
 \overline{f}(u) < \max_{w\in\{0,1\}^n} f(w) \quad \text{and}\quad
 \underline{f}(u) > \min_{w\in\{0,1\}^n} f(w).
\]
\end{theorem}

\begin{proof}
Since $1(v)\subseteq 1(u) \Leftrightarrow 0(u) \subseteq 0(v)$, we have
$u_i=0 \Rightarrow v_i=0$ for any $v$ such that $1(v)\subseteq 1(u)$.
Thus, fixing all $i\in 0(u)$ as $v_i=0$ and setting the remaining $v_i, i\in 1(u)$ free as
 \[
  \{f(v)\mid 1(v)\subseteq 1(u)\} =
  \{f(\alpha_1,\alpha_2,\dots,\alpha_n)\mid \alpha_i=0, i\in 0(u)\},
 \]
and taking the maximum and minimum in this set leads to the result in
 the theorem. Note that the above set is finite since $\alpha_i \in
 \{0,1\}, i \in 1(u)$.
\end{proof}

\begin{corollary}\label{gbnd}
 From Theorem \ref{maxmin}, we can have the
 upper and lower bounds of $f(I_{\mathcal{G}_n}(x'))$ at any $x' \in
 \mathcal{T}(x)$ for any function $f:\{0,1\}^n \to \mathbb{R}$ as 
\[
 \overline{f}(I_{\mathcal{G}_n}(x)) \geqslant f(I_{\mathcal{G}_n}(x')) \geqslant \underline{f}(I_{\mathcal{G}_n}(x)),
\]
because 
$1(I_{\mathcal{G}_n}(x'))\subseteq 1(I_{\mathcal{G}_n}(x))$ for $x' \in
 \mathcal{T}(x)$ from Lemma \ref{boolprop}.
\end{corollary}

\subsection{Morishita-Kudo Bounds and Their Examples}\label{mkb_sec}

Theorem \ref{maxmin} and Corollary \ref{gbnd} give us a general idea to
obtain pruning bounds for arbitrary function $f$ in the depth-first
traversal of an enumeration tree. Furthermore, applying Theorem
\ref{maxmin} to \textit{separable} functions
enables us to obtain the following easy-to-compute bounds that we call
\textit{Morishita-Kudo bounds}. Note that the target functions appeared
in the existing studies such as \cite{Morishita:2002,
Kudo:2005,Tsuda:2007,Saigo:2008a, Saigo:2009} are all separable
functions.

\begin{lemma}[Morishita-Kudo bounds]\label{mkbounds}
Assume that a real-valued function of $n$-dimensional Boolean vector
 $f:\{0,1\}^n \to \mathbb{R}$ is separable, meaning that there exists a set
 of $n$ functions $f_i:\{0,1\}\to\mathbb{R},i=1,2,\dots,n$ and
\[
 f(u_1,u_2,\dots,u_n) = \sum_{i=1}^n f_i(u_i), \quad u_i \in \{0,1\}.
\]
Then, for given $u=(u_1,u_2,\dots,u_n)\in\{0,1\}^n$, we have
\[
 \overline{f}(u) \geqslant f(v) \geqslant \underline{f}(u)
\]
for any $v = (v_1,v_2,\dots,v_n)\in\{0,1\}^n$ such that $1(v)\subseteq
 1(u)$, where
\begin{align*}
\underline{f}(u) &:= \sum_{i \in 1(u)} \min\{f_i(0),f_i(1)\} \,\,+
 \sum_{i \in 0(u)}f_i(0) \\
\overline{f}(u) &:= \sum_{i \in 1(u)} \max\{f_i(0),f_i(1)\} \,\,+
 \sum_{i \in 0(u)}f_i(0) 
\end{align*}
\end{lemma}

Thus, for separable functions, if we have some fixed $u \in
\{0,1\}^n$, then
we can limit the possible range of $f(v)$ for any $v$ such that
$1(v)\subseteq 1(u)$, and the upper and lower bounds for $f(v)$,
i.e. $\underline{f}(u)$ and $\overline{f}(u)$, are easy to compute 
just by comparing $f_i(0)$ and $f_i(1)$ for each $i\in 1(u)$.
Since $1(v)\subseteq 1(u)$, we have $v_i=0$ for $i\in 0(u)$ and
the amount of $\sum_{i \in 0(u)} f_i(0)$ is unchanged and cannot be
further improved.
Only elements that can differ are $v_i$s for $i \in
1(u)$, and therefore we have the maximum or
minimum of $f(v)$ for $v$ such that $1(v)\subseteq 1(u)$.

The original bounds by \cite{Morishita:2002,Kudo:2005} target only the
specific objective function of Adaboost, but many existing methods for
other objectives, such as \cite{Tsuda:2007,Saigo:2008a,Saigo:2009},
which share the common idea as described in Lemma \ref{mkbounds}. Indeed
most existing approaches are 
based on a branch-and-bound strategy with the Morishita-Kudo bounds for each
class of learning problems with separable loss functions. We see this
fact in the three examples below.

\subsubsection{Example 1: Find a subgraph that maximizes the gain}

Assume that we have a certain subgraph $x$ at hand. For the problem of
classification with $y_i \in \{-1,1\}$, consider the upper bound of the
gain attained by subgraph $z$ ($\supseteq x$)
\[
 \sum_{i=1}^n w_i y_i (2 I(z \subseteq g_i) -1)
\]
 by $-1/1$-valued features of $2 I(z \subseteq g_i) -1$,
where $w_i \geqslant 0$ \citep{Kudo:2005}. Then, from Lemma \ref{boolprop}, we have 
$x \subseteq z \Rightarrow 1(z) \subseteq 1(x)$. Thus we obtain a
Morishita-Kudo upper bound
 \begin{align*}
 \sum_{i=1}^n w_i y_i (2 I(z \subseteq g_i) -1) &\leqslant
 \sum_{i \in 1(x)} \max\{w_i y_i (2\cdot 0-1),w_i y_i (2\cdot 1-1)\} + \sum_{i \in 0(x)}
  w_i y_i (2\cdot 0-1) \\
 &=\sum_{i \in 1(x)} \max\{-w_i y_i,w_i y_i\} - \sum_{i \in 0(x)} w_i y_i
 =\sum_{i \in 1(x)} w_i - \sum_{i \in 0(x)} w_i y_i.
\end{align*}
In the original paper of \cite{Kudo:2005}, this bound was derived as
\begin{align*}
 \sum_{i=1}^n w_i y_i (2 I(z \subseteq g_i) -1) &=
 2 \sum_{i=1}^n w_i y_i I(z \subseteq g_i) - \sum_{i=1}^n w_i y_i \\
&\leqslant 2 \left\{
\sum_{i \in 1(x)} \max\{w_i y_i,0\}
+ \sum_{i \in 0(x)} w_i y_i \cdot 0
\right\} - \sum_{i=1}^n w_i y_i \\
&= 2 \sum_{i \in 1(x) \wedge y_i=1} w_i - \sum_{i=1}^n w_i y_i.
\end{align*}
We can confirm that these two are equivalent as 
$\sum_{i \in 1(x)} w_i - \sum_{i \in 0(x)} w_i y_i = 2 \sum_{i \in 1(x) \wedge y_i=1} w_i - \sum_{i=1}^n w_i y_i$
by subtracting the right-hand side from the left-hand side.

\subsubsection{Example 2: Find a subgraph that minimizes the
   classification error}

Assume that we have a certain subgraph $x$ at hand. For the problem of
classification with $y_i \in \{0,1\}$, consider the lower bound of the
weighted classification error attained by $z$ ($\supseteq x$)
\[
 \sum_{i=1}^n w_i I(I(z \subseteq g_i) \neq y_i),
\]
where $w_i \geqslant 0$. Similarly we have a Morishita-Kudo lower bound
\begin{align*}
 \sum_{i=1}^n w_i I(I(z \subseteq g_i) \neq y_i) &\geqslant
 \sum_{i \in 1(x)} \min\{w_i I(0 \neq y_i), w_i I(1 \neq y_i)\} + \sum_{i \in
 0(x)} w_i I(0 \neq y_i) \\
&=\sum_{i \in 0(x)} w_i I(0 \neq y_i) = \sum_{i \in 0(x) \wedge y_i=1} w_i.
\end{align*}
In other words, the classification error depends only on elements taking
$y_i=1$ but predicted as $I(x \subseteq g_i)=0$. Once the predictor
value becomes $0$, the objective function cannot be further improved by
searching $z$ such that $z \supseteq x$. We can conclude that for $x$
such that the amount $\sum_{i \in 0(x) \wedge y_i=1} w_i$ is large, the
possible value with $z$ by further searching $z$ such that $z \supseteq
x$ is upper bounded by this amount.

\subsubsection{Example 3: Find a subgraph that maximizes the correlation}

Assume we have a certain subgraph $x$ at hand. 
As seen in regression problems \citep{Saigo:2009}, consider the upper
bound of the correlation
between response $w_i \in \mathbb{R}$ and indicator $I(z\subseteq g_i)$
by subgraph feature $z$ ($\supseteq x$) 
\[
 \left| \sum_{i=1}^n w_i I(z\subseteq g_i)\right|.
\]
From $|a|=\max\{a,-a\}$, the larger amount of 
\[
 \sum_{i=1}^n w_i I(z\subseteq g_i)\qquad\text{or}\qquad
 -\sum_{i=1}^n w_i I(z\subseteq g_i)
\]
gives the upper bound of the correlation. Here we have a Morishita-Kudo
upper bound as
\[
 \sum_{i=1}^n w_i I(z\subseteq g_i) \leqslant
 \sum_{i \in 1(x)} \max\{w_i \cdot 0, w_i \cdot 1\} + \sum_{i \in 0(x)}
 w_i \cdot 0 = \sum_{i \in 1(x)} \max\{0, w_i\} = \sum_{i \in 1(x)
 \wedge w_i >0} w_i.
\]
We can always have $w_i I(x \subseteq g_i)=0$ for $i \in 0(x)$, and
rewrite as
\[
  \sum_{i=1}^n w_i I(z\subseteq g_i) \leqslant \sum_{i \in 1(x)
 \wedge w_i >0} w_i = \sum_{w_i>0} w_i I(x\subseteq g_i).
\]
Therefore, we can easily have the pruning bound in \cite{Saigo:2009} as
\[
 \left| \sum_{i=1}^n w_i I(z\subseteq g_i)\right| \leqslant \max\left\{
\sum_{w_i>0} w_i I(x\subseteq g_i), -\sum_{w_i<0} w_i I(x\subseteq g_i)
\right\}.
\]

\section{Learning Sparse Liner Models by Block Coordinate Gradient Descent}

Now we are back to our original problem, shown as Problem~\ref{prob1}. 
In this section, we describe the optimization process for Problem
\ref{prob1}, and in order to focus on this purpose, we first rewrite the problem in
terms of parameters to be estimated, $\theta := (\beta_0,\beta)$, as
nonsmooth optimization with a smooth part $f(\theta)$ and a nonsmooth
part $R(\theta)$ as follows.

\begin{problem}\label{prob2}
Find the minimizer
\[
 \theta^* = \arg\min_{\theta} F(\theta), \theta:=
 (\beta_0,\beta),
\]
where $F(\theta) := f(\theta)+R(\theta), \theta_0=\beta_0,
 \theta_1=\beta_1, \theta_2=\beta_2,\dots$, and
\begin{align*}
f(\theta) &:= 
\sum_{i=1}^n
L\bigl(y_i, \mu(g_i;\beta,\beta_0)\bigr) + \frac{\lambda_2}{2}
\|\beta\|_2^2\\
R(\theta) &:= \lambda_1 \|\beta\|_1 
\end{align*}
\end{problem}

Our basic premise behind Problem \ref{prob2} is that we need a
\textit{sparse} solution for $\beta$, meaning that only a few of the
coordinates $\beta_1,\beta_2,\dots$ are nonzero. Furthermore, whenever
we have intermediate solutions for 
searching the optimal $\theta^*$, we
also need to \textit{keep them sparse throughout
the entire optimization process}. This is simply because 
we have an intractably large number of parameters as $\beta_1,\beta_2,\dots$
even though Lemma \ref{finiteness} and \ref{enumprop1} suggest that they
are finite. 

In addition, these
parameters $\beta_1,\beta_2,\dots$ are associated with the
corresponding subgraph features $x_1,x_2,\dots$ that also need to be
estimated. In other words, in order to solve Problem \ref{prob2}, we
need to perform \textit{simultaneous learning} of the model coefficient
$\beta_i$ and corresponding subgraph feature $x_i$, which are correspondingly
paired as
\[
 \begin{array}{c}
  \beta_1 \\
  \updownarrow \\
  x_1
 \end{array}
 \begin{array}{c}
  \beta_2 \\
  \updownarrow \\
  x_2
 \end{array}\dots
 \begin{array}{c}
  \beta_i \\
  \updownarrow \\
  x_i
 \end{array}\dots
 \begin{array}{c}
  \beta_m \\
  \updownarrow \\
  x_m
 \end{array}
\]
where $m = |\mathcal{X}(\mathcal{G}_n)|$. Unlike standard situations in
optimization, this problem has two remarkably challenging difficulties:
\begin{itemize}
 \item We cannot estimate $\beta_i$ without finding the corresponding
       subgraph $x_i$. At the same time, however, it is practically
       impossible to find all subgraphs $x_1,x_2,\dots,x_m$ in advance.
 \item The number of parameters $m$ is unknown in advance and even not
       computable in most practical cases (intractably large) because of
       combinatorial explosion. 
\end{itemize}
On the other hand, we notice that we do not need to check subgraph
$x_i$ if we know $\beta_i=0$ beforehand.

Thus our idea for these two points is to start an iterative optimization
process with all zero initial values as
$\beta_i=0, i=1,2,\dots,m$, and search and trace subgraph features $x_i$ with
nonzero $\beta_i$ after each iteration, avoiding
the check for $x_j$ with $\beta_j=0$ as much as possible.
Moreover, when we first start with the zero vector, i.e. $\theta(0)=0$,
and then improve $\theta$ step by step as 
\[
 \theta(0) \to \theta(1) \to \theta(2) \to \dots \to \theta^*,
\]
we kept all $\theta(t)$ \textit{sparse}, meaning that only a few coordinates of
every $\theta(t)$ are nonzero. More precisely, we need to meet the
following two requirements.
\begin{enumerate}
 \item First, we need an iterative optimization
       procedure with keeping all intermediate $\theta(t)$ sparse such that
\[
 F(\theta(t)) > F(\theta(t+1))\quad t=0,1,2,\dots
\]
that implies convergence for any lower
 bounded function $F$ from the monotone convergence theorem. Furthermore
       we also need to make sure that the solution converges to the optimal solution
       $\theta^*$ as $\lim_{t \to \infty} F(\theta(t)) =
 F(\theta^*)$.
\item Second, 
       unlike standard cases, we need to perform the
       optimization with only the nonzero part of $\beta$ without even seeing
       the zero part of $\beta$ explicitly. In addition, we need to identify the
      nonzero part of $\theta$ at each step without checking subgraph $x_i$ that
      corresponds to $\beta_i=0$ as much as possible.
\end{enumerate}
We present our approach for solving these two issues, the first and
second points being in Section \ref{optim1} and in Section \ref{optim2},
respectively. 

\subsection{Tseng-Yun Class of Block Coordinate Gradient Descent}\label{optim1}

Suppose hypothetically that our features are explicitly given beforehand as in the
standard machine-learning situations. Then, there are many optimization
procedures available for the nonsmooth objective functions in Problem
\ref{prob2} such as
Proximal Gradient Method \citep{Fukushima:1981}, Iterative Shrinkage/Thresholding
\citep{Figueiredo:2003,Daubechies:2004}, Operator Splitting \citep{Combettes:2005}, Fixed Point Continuation
\citep{Hale:2008a,Hale:2008b}, FISTA \citep{Beck:2009}, SpaRSA
\citep{Wright:2009}, and glmnet \citep{Friedman:2007,Friedman:2010}.
Most of these methods are based on coordinate descent for efficiently
finding the optimal sparse solution. 

However, our problem is defined by the implicitly-given
subgraph-indicator features, and we need to simultaneously learn the necessary
subgraphs. To make things worse, since our problem has an 
intractably large number of parameters, we cannot even access to all
parameter values. Therefore, it is impossible to
simply apply these algorithms to our setting.
Moreover, we need to keep all intermediate
solutions sparse as well. Any generic solver based on interior-point
methods, which can generate non-sparse intermediate solutions, also does
not fit to our setting.

We thus extend a framework of block coordinate
descent with small nonzero coordinate blocks to our simultaneous learning of
features and parameters. \cite{Tseng:2009,Yun:2011} studied a class of block
coordinate gradient descent for nonsmooth optimization appeared in
Problem \ref{prob2}. They presented a fast algorithm and also established the
global and linear convergence under a local error bound
condition, as well as Q-linear convergence for 1-norm regularized convex
minimization. Moreover, when we
choose the coordinate block by a Gauss-Southwell-type rule, we can make
all the intermediate solutions sparse without breaking the theoretical
convergence property as we see below.

Their algorithm is based on gradient descent by applying 
local second-order approximation at the current $\theta(t)$ to only the smooth part $f(\theta)$ of the
objective function $F(\theta)=f(\theta)+R(\theta)$ as
\begin{align*}
 \min_{\theta} \left[ f(\theta)+R(\theta) \right] & =  \min_{\theta} \left[f(\theta) -
 f(\theta(t)) + R(\theta) \right]\\
& \approx \min_{\theta} \left[
\langle \nabla f(\theta(t)), \theta-\theta(t) \rangle + \frac{1}{2}
 \langle \theta - \theta(t), H(t) (\theta - \theta(t))\rangle
 + R(\theta)
\right]
\end{align*}
where $H(t)\succ 0$ is a positive-definite matrix approximating the Hessian
$\nabla^2 f(\theta(t))$. The main idea is to solve this local
minimization by block coordinate descent
instead of directly optimizing the original objective function by
coordinate descent, which may be viewed as a hybrid of
gradient projection and coordinate descent. The coordinate block to be
updated at each iteration are chosen in a Gauss-Southwell way, which can
be the entire coordinates or a small block of coordinates satisfying a
required condition.

The algorithm first computes a descent direction $d(t) := T(\theta(t))-\theta(t)$ where
\begin{equation}\label{Ttheta}
T(\theta(t)) := \arg \min_{\theta} \left[
\langle \nabla f(\theta(t)), \theta-\theta(t) \rangle
+ \frac{1}{2} \langle \theta - \theta(t), H(t) (\theta - \theta(t))
\rangle + R(\theta)
\right]
\end{equation}
which is optimized by block coordinate descent with a Gauss-Southwell
type rule of coordinate block for $d(t)$. We choose the coordinate block
by setting $d(t)_j = 0$ for $j$ such that $v(t) \|d(t)\|_{\infty} >
|d(t)_j|$ with some predefined scaling factor $v(t)$, which can adjust
the sparsity of intermediate solutions. This coordinate block satisfies
the Gauss-Southwell-r rule defined in \cite{Tseng:2009}, which ensures the
sufficient descent and global convergence. Next, the algorithm performs
one-dimensional minimization of stepsize $\alpha(t)$ for the update of
\[
 \theta(t+1) \leftarrow \theta(t) + \alpha(t)\, d(t), \quad 
\]
by a line search with the following Armijo rule: choose $\alpha_{\mathrm{init}}(t)>0$ and
let $\alpha(t)$ be the largest element of $\{\alpha_{\mathrm{init}}(t)\,
s^j\},j=0,1,\dots$, where $s$ is a scaling parameter such that $0 < s < 1$, satisfying
\begin{equation}\label{armijo}
 F(\theta(t)+\alpha(t)d(t)) \leqslant F(\theta(t)) + \alpha(t)\sigma \Delta(t)
\end{equation}
where $0 < \sigma < 1, 0 < \gamma < 1$, and
\[
 \Delta(t) := \langle \nabla f(\theta(t)),d(t) \rangle
+ \gamma \langle d(t), H(t) d(t)
\rangle + R(\theta(t)+d(t)) - R(\theta(t)).
\]
Note that since $\Delta(t)\leqslant (\gamma-1)\langle d(t), H(t) d(t)\rangle
<0 $, this rule guarantees that $F(\theta(t+1)) < F(\theta(t))$ after the update
of $\theta(t+1) \leftarrow \theta(t)+\alpha(t)d(t)$. The amount of
$\alpha(t)\sigma \Delta(t)$ is required for guaranteeing not only the
convergence to some point but also the convergence to the optimal
solution (See their original papers for theoretical details). This
Armijo rule is based on the following lemma (Lemma 1 in
\cite{Tseng:2009}).

\begin{lemma}[Tseng and Yun, 2009]
For any $\theta$ and $H\succ 0$, let $g=\nabla f(\theta)$ and $d_J=T(\theta)-\theta$ such that $d_j=0$
 for $j$ in the index subset $J$. Then
\begin{align*}
& F(\theta+\alpha d_J) \leqslant F(\theta) + \alpha\left(
\langle g,d_J \rangle + R(\theta+d_J) - R(\theta)
\right) + o(\alpha) \quad \forall \alpha \in (0,1]\\
& \langle g,d_J \rangle + R(\theta+d_J) - R(\theta) \leqslant - \langle d_J,
 H d_J \rangle.
\end{align*}
\end{lemma}

\subsection{Tracing Nonzero Coefficient Block of Subgraph Indicators}\label{optim2}

Now we consider how we can use the Tseng-Yun block coordinate gradient descent
in our subgraph-search-required situations of Problem \ref{prob1} (or
equivalently Problem \ref{prob2}). Remember that the parameter
space is intractably large due to the size of the accompanying subgraph
space $\mathcal{X}(\mathcal{G}_n)$. Each $j$-th coordinate of $\theta(t)$ is
associated with the corresponding subgraph feature $x_j \in
\mathcal{X}(\mathcal{G}_n)$. Thus, we cannot even have $\theta(t)$
explicitly, and we need to perform the optimization process
with tracing only the nonzero part of $\theta(t)$ without even seeing
the zero part explicitly. Since we start with setting all coordinates of
initial $\theta(0)$ to zero, it is enough to consider the update rule to
compute $\theta(t+1)$ from the current $\theta(t)$ at hand. More
precisely, we need a procedure to obtain the nonzero part of
$\theta(t+1)$ from the nonzero part of $\theta(t)$ at hand.

Suppose that we have all nonzero coordinates of $\theta(t)$. 
The Tseng-Yun's iterative procedure requires the following
computations at each step of $\theta(t) \to \theta(t+1)$ until convergence:
\begin{enumerate}
 \item[] \textbf{Step 1.} Compute $T(\theta(t))$ by coordinate descent for \eqref{Ttheta}.
 \item[] \textbf{Step 2.} Compute a descent direction by $d(t) = T(\theta(t)) - \theta(t)$.
 \item[] \textbf{Step 3.} Set Gauss-Southwell-r block by $d(t)_j=0$ for
       $\{j \mid v(t) \|d(t)\|_{\infty} > |d(t)_j|\}$.
 \item[] \textbf{Step 4.} Do a line search for $\alpha(t)$ with the
	 modified Armijo rule of \eqref{armijo}.
 \item[] \textbf{Step 5.} Update the parameter by $\theta(t+1) \leftarrow \theta(t) + \alpha(t) d(t)$.
\end{enumerate}
Since we already have all nonzero coordinates of $\theta(t)$, the
determining step to detect the nonzero part of $\theta(t+1)$ is
\textbf{Step 1}. Indeed, if we obtain the nonzero part of $T(\theta(t))$,
\textbf{Step 2} to \textbf{5} follow just as in familiar optimization
situations (See Figure \ref{thetatree} for a schematic example). 


\begin{figure}[t]
\centering
\includegraphics[width=140mm]{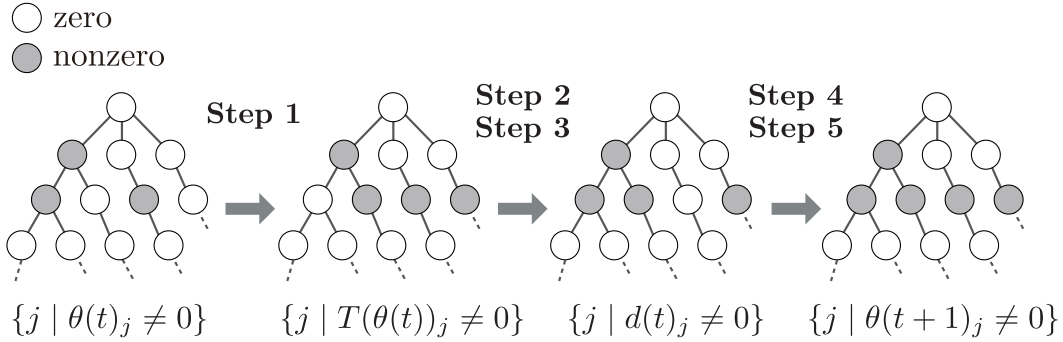}
\caption{
A schematic example of corresponding nonzero coordinates of $\theta(t)$,
 $T(\theta(t))$, $d(t)$, and $\theta(t+1)$ in the enumeration tree.
The nonzero coordinates of $\theta(t+1)$ is
 computed from those of $\theta(t)$ through $T(\theta(t))$ and
 $d(t)$, and identifying $T(\theta(t))_j \neq 0$ is the determining step.
Since $d(t)=T(\theta(t))-\theta(t)$, we have
 $\{j\mid d(t)_j
 \neq 0\} \subseteq \{j\mid \theta(t)_j \neq 0\} \cup
 \{j\mid T(\theta(t))_j \neq 0\}$. 
Since $\theta(t+1)=\theta(t)+\alpha(t)d(t)$, we can also see 
 $\{j\mid 
 \theta(t+1)_j \neq 0\} 
 \subseteq \{j\mid \theta(t)_j \neq 0\} \cup
 \{j\mid T(\theta(t))_j \neq 0\}$.
} \label{thetatree}
\end{figure}

We thus focus on how to identify the nonzero indices $\{j
\mid T(\theta(t))_j \neq 0 \}$ from the current $\theta(t)$ at hand
(\textbf{Step 1}). Note that \textbf{Step 1} is based on coordinate
descent and the computation of each $i$-th coordinate $T(\theta(t))_j$
can be separately estimated. Also note that \textbf{Step 3},
 \textbf{4} and \textbf{5} can be carried out only after \textbf{Step 1}
 and \textbf{2} for all nonzero coordinates are obtained because \textbf{Step
 3} requires  $\|d(t)\|_\infty$. 

The starting point to realize \textbf{Step 1} in our ``graph'' situations
is the following lemma, which shows that $T(\theta(t))_j$ for the $j$-th
coordinate has the closed-form solution.

\begin{lemma}\label{cond1}
In Problem \ref{prob1}, when we solve \eqref{Ttheta} by coordinate
 descent, the following closed-form solution exists for each $j=1,2,\dots$:
\[
 T(\theta(t))_j = 
\begin{cases}
-H(t)_{jj}^{-1}(b_j(t) + \lambda_1) & b_j(t) < -\lambda_1 \\
-H(t)_{jj}^{-1}(b_j(t) - \lambda_1) & b_j(t) > \lambda_1 \\
0 & |b_j(t)| \leqslant \lambda_1
\end{cases}
\]
where 
\[
 b_j := 
\sum_{i=1}^n \frac{\partial L(y_i,\mu(g_i;\theta(t)))}{\partial
 \theta(t)_j} + (\lambda_2 - H(t)_{jj}) \theta(t)_j.
\]
\end{lemma}

\begin{proof}
 See Appendix A1.
\end{proof}

Combined with the structured search space
of $\mathcal{X}(\mathcal{G}_n)$ which is equal to the enumeration tree
$\mathcal{T}(\mathcal{G}_n)$,
Lemma \ref{cond1} provides a way to examine if $T(\theta(t))_k=0$ for
unseen $k$ such that $x_k \in \mathcal{T}(x_j)$:
When we know $T(\theta(t))_k=0$ for all unseen $x_k \in
\mathcal{T}(x_j)$, we can skip the check of all subgraphs in the
subtree below $x_j$, i.e. $x_k \in \mathcal{T}(x_j)$

As the lemma claims, $T(\theta(t))_k=0$ is controlled by whether
$|b_k(t)| \leqslant \lambda_1$ or not. If we know the largest possible
value $b_j^*$ of $|b_k(t)|$ for any $k$ such that $x_k \in
\mathcal{T}(x_j)$ and also know that $b_j^* \leqslant \lambda_1$, 
then we can conclude $T(\theta(t))_k=0$ for all of such $k$s.
To consider how to obtain this bound $b_j^*$, let $b_k^{(1)}(t)$ and
$b_k^{(2)}(t)$ be the first and second terms of $b_k(t)$, respectively, as
\[
 b_k^{(1)}(t) := \sum_{i=1}^n \frac{\partial L(y_i,\mu(g_i;\theta(t)))}{\partial
 \theta(t)_k}, \qquad
 b_k^{(2)}(t) := (\lambda_2 - H(t)_{kk}) \theta(t)_k.
\]
Then, since
\begin{equation}\label{ineq}
\text{$|b^{(1)}+b^{(2)}| \leqslant
 \max\{\overline{b}^{(1)}+\overline{b}^{(2)},-\underline{b}^{(1)}-\underline{b}^{(2)}\}$ for any
 $\underline{b}^{(1)}\leqslant b^{(1)} \leqslant \overline{b}^{(1)}, \underline{b}^{(2)}\leqslant b^{(2)} \leqslant \overline{b}^{(2)}$},
\end{equation}
we can obtain the bounds for $|b_k(t)|$ if we have the individual bounds
for $b_k^{(1)}$ and $b_k^{(2)}$. 
Since $b_k^{(1)}(t)$ is
separable as we see later in Appendix A2, we can have Morishita-Kudo
bounds by Lemma \ref{mkbounds}. 
On the other hand, the second term $b_k^{(2)}(t)$ is not the case, to
which Morishita-Kudo bounds can be applied. However, since we already have
all $\theta(t)_j$ for the
nonzero indices $\{j\mid \theta(t)_j \neq 0\}$, we already have
$H(t)_{kk}$ and $\theta(t)_k$ for $\theta(t)_k \neq 0$. We might not have the value
of $H(t)_{kk}$ for $\theta(t)_k = 0$, but in this case we have
$b_k^{(2)}=0$ regardless of the value of $H(t)_{kk}$. Then,
provided that we have the index-to-set mapping 
\[
 j \mapsto \{k \mid x_k \in \mathcal{T}(x_j), \theta(t)_k \neq 0\}
\]
at each $j$, we can also have the exact upper and lower bounds for
$|b_k(t)|$ such that $x_k \in \mathcal{T}(x_j)$. We will see in the next
section how to construct this mapping efficiently. 
By Lemma \ref{cond1} and the technique we call
\textit{depth-first dictionary passing} described in the next
subsection, we have the following result. Note that this also confirms
that we can control the sparsity of $\theta(t)$ by the parameter $\lambda_1$.

\begin{theorem}\label{cond2}
Suppose we have $x_j \in \mathcal{T}(\mathcal{G}_n)$. Then, for any $x_k \in
 \mathcal{T}(x_j)$, there exist upper and lower bounds
\begin{align*}
 & \underline{L}_j(t) \leqslant \sum_{i=1}^n \frac{\partial L(y_i,\mu(g_i;\theta(t)))}{\partial
 \theta(t)_k} \leqslant \overline{L}_j(t)\\
 & \underline{B}_j(t) \leqslant (\lambda_2 - H(t)_{kk}) \theta(t)_k \leqslant \overline{B}_j(t),
\end{align*}
that are only dependent on $x_j$, and $T(\theta(t))_k=0$ if $
 \max\{\overline{L}_j(t)+\overline{B}_j(t),-\underline{L}_j(t)-\underline{B}_j(t)\}
 \leqslant \lambda_1$.
\end{theorem}

\begin{proof}
 See Appendix A2.
\end{proof}

\begin{remark}
If we observe $
 \max\{\overline{L}_j(t)+\overline{B}_j(t),-\underline{L}_j(t)-\underline{B}_j(t)\}
 \leqslant \lambda_1$ at $x_j$, we can conclude that there are no $x_k
 \in \mathcal{T}(x_j)$ such
 that $T(\theta)_k \neq 0$, and therefore prune
 this entire subtree $\mathcal{T}(x_j)$ in search for $T(\theta)_k \neq 0$.
\end{remark}

\subsection{Depth-First Dictionary Passing}\label{sec_dfcp}

\begin{figure}[t]
\centering
\includegraphics[height=0.23\textheight]{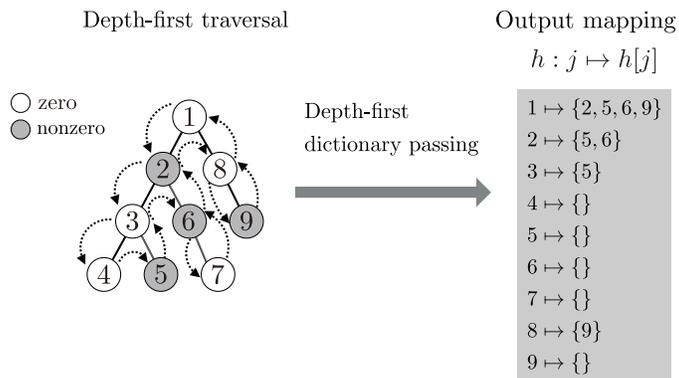}
\caption{
An example of depth-first dictionary passing. Depth-first
 dictionary passing is a procedure to produce the
 mapping $h$ on the right after the depth-first traversal of the tree on
 the left. Each $h[j]$ is the indices of nonzero nodes below node $j$.
} \label{dfcp}
\end{figure}

In this subsection, we describe how to obtain $\underline{B}_j(t)$ and
$\overline{B}_j(t)$ in Theorem \ref{cond2} which are needed for
computing $\theta(t+1)$ from $\theta(t)$. 
Since 
we already have the nonzero coordinates of $\theta(t)$ for the index subset
$\{j\mid\theta(t)_j\neq 0\}$ at hand, these bounds are defined simply
as the maximum and minimum of the finite candidates: 
\begin{align*}
\underline{B}_j(t) &:= \min_{l \in \{k\mid x_k \in \mathcal{T}(x_j)\}}
 (\lambda_2 - H(t)_{ll}) \theta(t)_l = \min \left(
0, \min_{l \in \{k\mid x_k \in \mathcal{T}(x_j), \theta(t)_k \neq 0\}}
 (\lambda_2 - H(t)_{ll}) \theta(t)_l
\right),\\
\overline{B}_j(t) &:= \max_{l \in \{k\mid x_k \in \mathcal{T}(x_j)\}}
 (\lambda_2 - H(t)_{ll}) \theta(t)_l = \max \left(
0, \max_{l \in \{k\mid x_k \in \mathcal{T}(x_j), \theta(t)_k \neq 0\}}
 (\lambda_2 - H(t)_{ll}) \theta(t)_l
\right).
\end{align*}
We can compute these values in a brute-force manner by checking $x_j
\subseteq x_k$ for all $k$ such that $\theta(t)_k \neq 0$, 
but this is practically inefficient due to subgraph isomorphism test.

What we need instead is an efficient way to obtain, for each $j$, the index subset
\begin{equation}\label{tmap}
\{k\mid x_k \in \mathcal{T}(x_j), \theta(t)_k \neq 0\}.  
\end{equation}
Recall that in order to obtain $\{j\mid \theta(t)_j \neq 0\}$, we
traverse the enumeration tree $\mathcal{T}(\mathcal{G}_n)$ to
find all $j$ such that $T(\theta(t-1))_j \neq 0$ using the bounds 
in Theorem \ref{cond2} (Figure \ref{thetatree}). During this traversal,
we can also record, for every visited node $x_j$, the mapping
\begin{equation}\label{maps}
 j \mapsto \{k\mid x_k \in \mathcal{T}(x_j), T(\theta(t-1))_k \neq 0\}.
\end{equation}
If we have this mapping, then we can recursively obtain
\eqref{tmap} for each $j$ as
\begin{align}
&\{k\mid x_k \in \mathcal{T}(x_j), \theta(t)_k \neq 0\}\nonumber \\ 
&= \left(\{k\mid x_k \in \mathcal{T}(x_j), T(\theta(t-1))_k \neq 0\}\cup
\{k\mid x_k \in \mathcal{T}(x_j), \theta(t-1)_k \neq 0\}\right)\nonumber
 \\ 
&\hspace{2em}\cap \{k\mid\theta(t)_k \neq 0\}. \label{recursive}
\end{align}
Hence in what follows, we present the procedure, which we
call \textit{depth-first dictionary passing}, to build the mapping of
\eqref{maps} over all necessary $j$. Figure \ref{dfcp} shows an example
of depth-first dictionary passing for the search tree on the left (The
output is shown on the right).

First, we define
$\mathcal{T}_{\mathrm{visited}}(\mathcal{G}_n)$ as the tree
consisting of all visited nodes during the depth-first traversal of
$\mathcal{T}(\mathcal{G}_n)$ with pruning based on Theorem
\ref{cond2}. Note that $\mathcal{T}_{\mathrm{visited}}(\mathcal{G}_n)$
is a subtree of $\mathcal{T}(\mathcal{G}_n)$, and it holds that $x_j \in
 \mathcal{T}_{\mathrm{visited}}(\mathcal{G}_n)$ for all $x_j$ such that
 $T(\theta(t))_j \neq 0$. We then define an auxiliary mapping as 
\begin{equation}\label{dict}
h'(t) := \left\{ 
 (j \mapsto \{k\mid x_k \in \mathcal{T}(x_j), T(\theta(t))_k \neq 0\})
\mid x_j \in \mathcal{T}_{\mathrm{visited}}(\mathcal{G}_n)
\right\}
\end{equation}
with four operations of \textsc{Get}, \textsc{Put}, \textsc{Keys}, and
\textsc{DeleteKey}: For a mapping $h: j \mapsto z$, \textsc{Get} returns
the target $z$ of indicated $j$, which we write by $h[j]$. 
\textsc{Put} registers a new pair $(j\mapsto z)$ to $h$, which we write
simply by  $h[j] \leftarrow z$. \textsc{Keys} returns
all registered keys in $h$ as $\text{\textsc{Keys}}(h) = \{j \mid (j \mapsto z)
\in h\}$. \textsc{DeleteKey} deletes the pair $(j\mapsto z)$
 indicated by $j$ from $h$ as $\text{\textsc{DeleteKey}}(h,j) = h -
 \{(j\mapsto z)\}$.

During the depth-first traversal of $\mathcal{T}(\mathcal{G}_n)$ at time
$t$, we keep a tentative mapping, $h_{\mathrm{tmp}}(t)$, to
build $h'(t)$ of \eqref{dict} at the end. We start with an empty
$h_{\mathrm{tmp}}(t)$. Then we update and pass $h_{\mathrm{tmp}}(t)$ to
the next node of the depth-first traversal. In the pre-order operation,
if we encounter $x_j \in \mathcal{T}_{\mathrm{visited}}(\mathcal{G}_n)$
such that $T(\theta(t))_j \neq 0$, we add this $j$ to all
$h_{\mathrm{tmp}}[i]$ for $i \in
\text{\textsc{Keys}}(h_{\mathrm{tmp}})$ by $h_{\mathrm{tmp}}[i]
\leftarrow  h_{\mathrm{tmp}}[i]\cup\{j\}$. This informs all the
 ancestors that $T(\theta(t))_j \neq 0$. Then, we register $j$ itself to
 $h_{\mathrm{tmp}}$ by initializing as $h_{\mathrm{tmp}}[j] \leftarrow \{\}$. 
In the post-order operation, if we have $h_{\mathrm{tmp}}[j] \neq \{\}$,
 then it implies that $x_j$ has descendants $x_k \in \mathcal{T}(x_j)$
 such that $T(\theta(t))_k \neq 0$ and $k\in h_{\mathrm{tmp}}[j]$.
Thus, this is an element of our target set $h'(t)$, and we finalize the $j$-th
element $h'(t)[j]$ by setting $h'(t)[j] \leftarrow
h_{\mathrm{tmp}}[j]$. At this point, we also remove $j$ from
$h_{\mathrm{tmp}}$ by \textsc{DeleteKey}$(h_{\mathrm{tmp}},j)$ because
$j$ cannot be an ancestor of any forthcoming nodes in the subsequent
depth-first traversal after $x_j$. In this way, we can obtain $h'(t)$
defined in \eqref{dict} at the end of traversal.

The depth-first dictionary passing finally gives us the mapping \eqref{tmap} as
\[
h(t) := \left\{ 
 (j \mapsto \{k\mid x_k \in \mathcal{T}(x_j), \theta(t)_k \neq 0\})
\mid x_j \in \mathcal{T}_{\mathrm{visited}}(\mathcal{G}_n)
\right\}
\]
from $h'(t-1), h(t-1)$ and the nonzero coordinates of $\theta(t)$ by
using a recursive formula of \eqref{recursive}. This also gives
$\underline{B}_j(t)$ and $\overline{B}_j(t)$ in Theorem \ref{cond2}
which are needed for computing $\theta(t+1)$ from $\theta(t)$.

\begin{figure}[t]
\centering
\includegraphics[width=150mm]{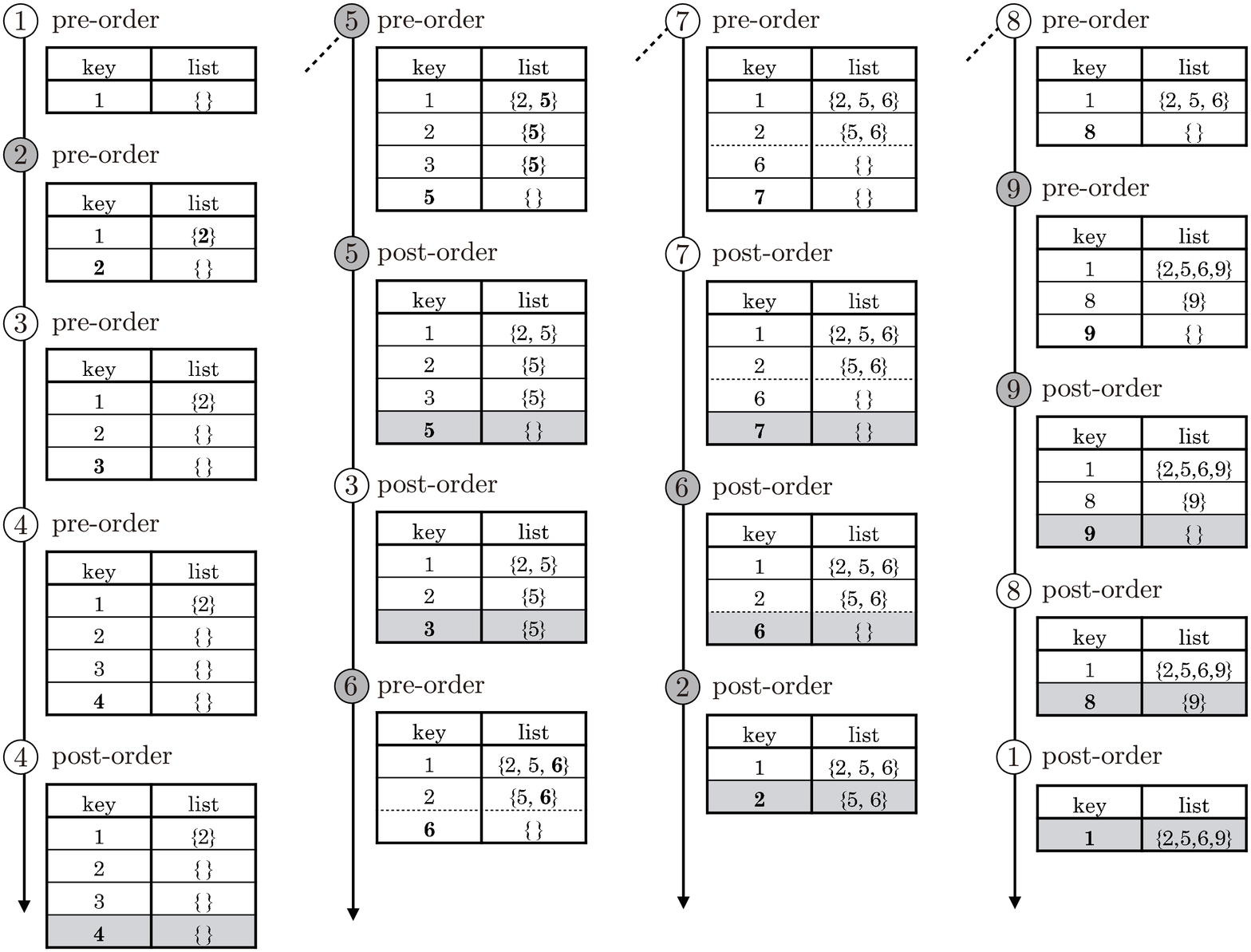}
\caption{
The transition of tentative mapping $h_{\mathrm{tmp}}$ in the
 depth-first dictionary passing in Figure \ref{dfcp}. The order follows
 the depth-first order of traversal, and each shaded mapping at
 post-order operations constitutes the final output $h$. At the end of
 traversal, we have the mapping $h$ on the right in Figure \ref{dfcp} as
 the collection of all shaded mappings.
} \label{dfcp_flow}
\end{figure}

Figure \ref{dfcp_flow} shows an example of building $h_{\mathrm{tmp}}$
(each key-list table) and registration to $h'$ (shown as the shaded
colored key-list pairs) for the case of Figure \ref{dfcp}. First, we
start a depth-first traversal from node $1$. In the pre-order operation at
node $1$, we just register key $1$ to $h_{\mathrm{tmp}}$ as
$h_{\mathrm{tmp}}[1] \leftarrow \{\}$. The next node $2$ is shaded, which
indicates that $T(\theta(t))_2 \neq 0$, then we inform node $2$ to all
ancestors by $h_{\mathrm{tmp}}[1] \leftarrow
h_{\mathrm{tmp}}[1] \cup \{2\}$ and also initialize as
$h_{\mathrm{tmp}}[2] \leftarrow \{\}$. At node $3$ and $4$, we just add
$h_{\mathrm{tmp}}[3] \leftarrow \{\}$ and $h_{\mathrm{tmp}}[4]
\leftarrow \{\}$, but at node $4$, we also need the post-order operation of
registering the mapping of $4$ to $h'$ and then removing key $4$. As a result,
at the shaded node of $5$, we have keys of $1,2$
and $3$ in $h_{\mathrm{tmp}}$. In the pre-order operation at node
$5$, we add $5$ to all of these as $h_{\mathrm{tmp}}[1] \leftarrow
h_{\mathrm{tmp}}[1]\cup\{5\}$, $h_{\mathrm{tmp}}[2] \leftarrow
h_{\mathrm{tmp}}[2]\cup\{5\}$ and $h_{\mathrm{tmp}}[3] \leftarrow
h_{\mathrm{tmp}}[3]\cup\{5\}$. In this way, we keep
$h_{\mathrm{tmp}}$, and obtain $h'$ at the end (as the set of the shaded
key-list pairs in tables of Figure \ref{dfcp_flow}).

\subsection{Algorithm}

Figure \ref{algo} shows the pseudocode for the entire procedure.

\IncMargin{1em}

\begin{algorithm}
\hrule
\vspace*{0.5ex}
\SetKwInput{KwAlgo}{Algorithm}
\KwAlgo{}
\vspace*{0.5ex}
\hrule
\vspace*{0.5ex}
$\theta(0) \leftarrow 0$\;
Build empty $h, h'$\;
\For{$t=0,1,2,\dots$}{
 Build an empty $hes, h_{\mathrm{tmp}}$\;
 \lForEach{$i \in \{i\mid\theta(t)_i\neq 0\}$}{$hes[i] \leftarrow
 H(t)_{ii}$\;}
 \ForEach{$x_j \in \mathcal{T}(\mathcal{G}_n)$ in the depth-first traversal}{
 \Begin(\emph{pre-order operation}){
 Compute $T(\theta(t))_j$ \tcp*[l]{by Lemma \ref{cond1}}
 \If{$T(\theta(t))_j\neq 0$}{
 \lForEach{$i\in$\textsc{Keys}$(h_{\mathrm{tmp}})$}{$h_{\mathrm{tmp}}[i] \leftarrow h_{\mathrm{tmp}}[i] \cup \{j\}$\;}
 }
 $h_{\mathrm{tmp}}[j] \leftarrow \{\}$\;
 $h[j] \leftarrow (h[j]\cup h'[j]) \cap \{i\mid \theta(t)_i \neq 0\}$
\tcp*[l]{by formula \eqref{recursive}}
 \eIf{$h[j] = \{\}$}{
 $\underline{B}_j \leftarrow 0, \overline{B}_j \leftarrow 0$\;
}{
 $\overline{B}_j \leftarrow \max\left\{\max_{k \in h[j]} \{
 (\lambda_2- hes[k]) \theta(t)_k\},0\right\}$\;
 $\underline{B}_j \leftarrow \min\left\{\min_{k \in h[j]} \{ (\lambda_2- hes[k]) \theta(t)_k\},0\right\}$\;
 }
 Compute $\overline{L}_j, \underline{L}_j $
\tcp*[l]{by Lemma \ref{mkbounds}}
 \eIf(\texttt{// by Theorem \ref{cond2}}){$\max\{\overline{L}_j+\overline{B}_j,-\underline{L}_j-\underline{B}_j\} \leqslant \lambda_1$
}{\emph{Prune} $\mathcal{T}(x_j)$\;}{\emph{Visit children of $x_j$}\;}
}
 \Begin(\emph{post-order operation}){
 \If{$h_{\mathrm{tmp}}[j] \neq \{\}$ }{
 $h'[j] \leftarrow h_{\mathrm{tmp}}[j]$\;
 }
 \textsc{DeleteKey}$(h_{\mathrm{tmp}},j)$\;
 }
}
$d(t)_i \leftarrow T(\theta(t))_i - \theta(t)_i$ for $i \in \{i\mid
 T(\theta(t))_i \neq 0 \vee \theta(t)_i \neq 0\}$\;
\emph{Gauss-Southwell-r: $d(t)_i \leftarrow
 0$ for $i$ such that $|d(t)|_i \leqslant v(t)\cdot\|d(t)\|_\infty$}\;
\emph{Armijo:} $\theta(t+1) \leftarrow \theta(t) + \alpha d(t)$ \tcp*[l]{by formula \eqref{armijo}}
\emph{Convergence test:} \lIf{\emph{$\|H(t)d(t)\|_\infty \leqslant \epsilon$}}{quit\;}
}
\vspace*{0.5ex}
\hrule
\caption{The entire algorithm for solving Problem \ref{prob1}.}\label{algo}
\end{algorithm}\DecMargin{1em}

\section{Numerical Results}

The proposed approach not only estimates the model parameters, but also
simultaneously searches and selects necessary subgraph features from
all possible subgraphs. This is also the case with existing approaches
such as Adaboost \citep{Kudo:2005} and LPBoost \citep{Saigo:2009}.

In this section, we thus numerically investigate the following three points
to understand the difference from these previously developed methods.
\begin{itemize}
 \item \textit{Convergence properties}: Contribution of the selected
       subgraph features at each iteration to the convergence to the
       optimal solution.
 \item \textit{Selected features}: Difference in the final selected
       subgraph features of three methods.
 \item \textit{Search-tree size}: Enumeration tree size at each
       iteration for searching necessary subgraph features. 
\end{itemize}

In order to systematically analyze these three points, we develop a
benchmarking framework where we can control the property, size, and
number of graph samples as well as we can measure the training and test
error of the model at each iteration. We
generate two sets of graphs by probabilistically combining a
small random subgraphs as \textit{unobserved} discriminative features,
and prepare a binary classification task of these two sets in a
supervised-learning fashion (the details are described in Section
\ref{datagen}). This generative model allows us to simulate the situation
where observed graphs have discriminative subgraph features behind, 
and we can compare the selected subgraph features by each learning
algorithm to these unobserved ``true'' features. In addition, we can
also evaluate not only the training error but also the test error. 
These two points are required since our goal here is to understand
the basic properties for convergence and subgraph feature selection.
For this binary classification task, we compare the following three methods. We
set the same convergence tolerance of $10^{-3}$ for both logistic
regression and LPBoost.

\begin{enumerate}
 \item \textbf{1-norm penalized logistic regression for graphs (with the 
       proposed learning algorithm)}. 
Logistic regression is one of the most standard methods for binary classification,
       which is often considered as a baseline for performance evaluation.
       We optimize the logistic regression with 1-norm regularization by
       the proposed algorithm with $\lambda_2 =
       0$ and 
\begin{equation}\label{logregloss}
 L(y,\mu) = y \log (1+\exp(-\mu)) + (1-y) \log (1+\exp(\mu)), \quad y
 \in \{0,1\}
\end{equation}
(or equivalently $L(y,\mu) = \log(1+\exp(- y \mu)), y \in
       \{-1,1\}$). For the loss function of \eqref{logregloss}, we have
       the derivative and the Hessian as
\[
 \frac{\partial L(y,\mu)}{\partial \mu} = \frac{1}{1+\exp(-\mu)} - y,\qquad
 \frac{\partial^2 L(y,\mu)}{\partial^2 \mu} = \frac{1}{1+\exp(-\mu)} \cdot \frac{1}{1+\exp(\mu)},
\]
respectively. 
Moreover, according to the example of \cite{Yun:2011}, we set
\[
  H(t)_{jj} := \min\left\{\max\{\nabla^2 f(\theta(t))_{jj},10^{-10}\},10^{10}\right\}
\]
and
\[
 \sigma = 0.1, c = 0.5, \gamma = 0, \alpha_{\mathrm{init}}(0) = 1,
 \alpha_{\mathrm{init}}(t) =
       \min\left\{\frac{\alpha(t-1)}{c^5},1\right\}, v(t)=0.9.
\]
Note that at each iteration, our optimization method adds multiple
       subgraph features at once to the selected feature set, which is
       different from the following two methods.
 \item \textbf{Adaboost for graphs \citep{Kudo:2005}}. This is a
       learning method by a sparse linear model over all possible
       subgraphs by adding the single best subgraph feature to
       improve the current model at each step, by using the framework of
       Adaboost. Adaboost can be asymptotically
       viewed as a greedy minimization of the exponential loss function.
 \item \textbf{LPBoost for graphs \citep{Saigo:2009}}. Adaboost adds a
       single feature at each iteration, and does not update
       the indicator coefficients already added to the current model.
       By contrast, LPBoost updates all previous coefficients at each
       iteration of adding a new single feature. This property is known as
       the totally-corrective property which can accelerate the convergence.
       LPBoost minimizes the hinge loss function with 1-norm
       regularization. Since the hinge loss is not twice
       differentiable, LPBoost is complementary to our framework.
\end{enumerate}
We emphasize that classification performance depends on the
objective function and has nothing to do with the optimization
method. Our interest is not necessarily in comparing the performance and
property of Logistic regression, Adaboost, and LPBoost, which have long
been discussed in the machine-learning community.
Rather, as we already mentioned above, we are interested in the
convergence properties, selected subgraph features, and search-tree
sizes along with adaptive feature learning from all possible subgraph
indicators.

L1-LogReg was implemented by C++ entirely from the
scratch. For Adaboost and LPBoost, we used the implementations in
C++ by \cite{Saigo:2009} which are faster than other implmentations. The
source code was obtained from author's
website\footnote{\url{http://www.bio.kyutech.ac.jp/~saigo/publications.html}}, 
and we added a few lines carefully to count the number of
visited nodes. We also set minsup=1 and maxpat=$\infty$ in the original
code, to run Adaboost and LPBoost over all possible subgraph features. 

\subsection{Systematic Binary Classification Task}\label{datagen}

Figure \ref{randgen} presents the procedure for generating two set of
graphs\footnote{In the pseudo code of Figure \ref{randgen}, the subroutine
 \texttt{MinDFSCode}$(g)$ computes a canonical representation of graph
 $g$ called the minimum DFS code (See \cite{Yan:2002} for technical
 details). }.
These two sets correspond to graph sets (i) and (ii) in Figure
\ref{matb} which are generated by a consistent probabilistic rule: each
graph in (i) includes each subgraph feature in subgraph set (A) with a
probability of $p_1$ and similarly each feature in (B) with $q_1$; Each
graph in (ii) includes each subgraph feature in (A) with a probability
of $p_2$ and each feature in (B) with $q_2$ (as summarized on the right in
Figure \ref{matb}), where a seed-graph pool of (A) and (B) are also
generated randomly.
When we set $p_1 > q_1$ and $q_2 > p_2$, 
we can assume $V$ and $W$ dominant subgraph features that can
discriminate between (i) and (ii). 
By using these two sets (i) and (ii) as positive and negative samples
respectively, we define a systematic binary classification task with
control parameters $V, W, N, M, p_1, p_2, q_1, q_2$ and the Poisson mean 
parameter. 
This procedure, which was inspired by \cite{Kuramochi:2004}, first
generates $V+W$ random graphs as $X:=\{x_1,x_2,\dots,x_{V+W}\}$ for a
seed-graph pool, and then generates $N+M$ output graphs by selecting and
combining these seed graphs in $X$ with the probabilistic rule: 
as indicated in Figure \ref{matb}, the probability that output graph $i$
contains subgraph feature $j$ or not ($I_{i,j}=1$ or $0$) is determined
block by block. 
To generate each output graph, the selected seed subgraphs are combined by
the procedure \texttt{combine} in Figure \ref{combine} which finally 
produces the output graphs (i) and (ii) having the data matrix in Figure
\ref{matb}\footnote{\cite{Kuramochi:2001} proposed
another way to combine the seed graphs to maximize the overlapped
subgraphs. In either way, the procedure in Figure \ref{randgen} produces
the data matrix as in Figure \ref{matb}.}.
The mean parameter for the Poisson distribution is set to $3$,
and the number of node and edge labels to $5$ throughout the numerical
experiments. 

\begin{figure}[h]
\centering
\includegraphics[width=\textwidth]{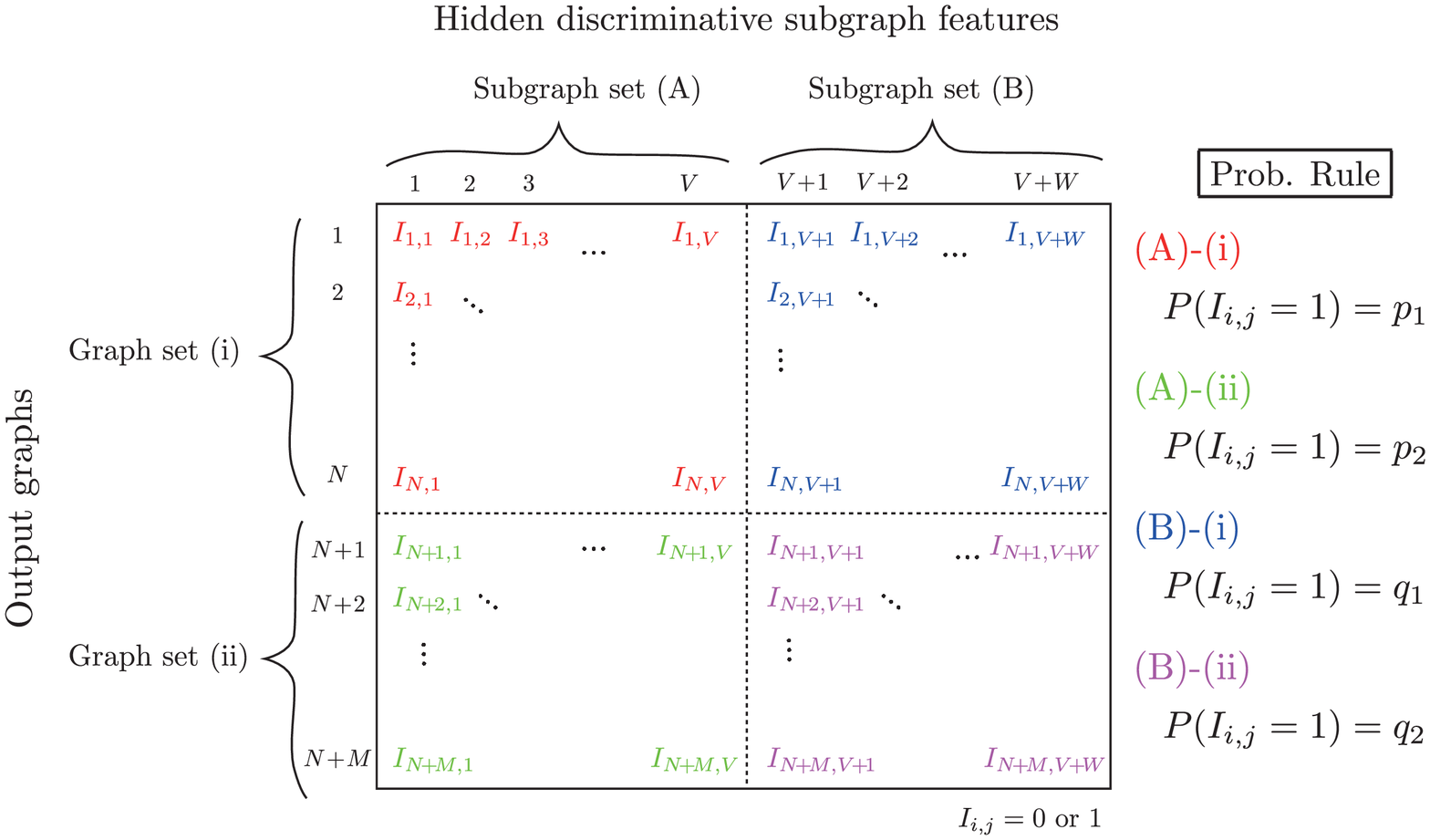}
\caption{
The data matrix: two graph sets (i) and (ii) which are generated by the
 procedure in Figure \ref{randgen}. 
These graphs are generated by combining graphs in random subsets of a
 seed-graph pool of (A) and (B) with the procedure in Figure
 \ref{combine}. Each subset is randomly generated, following the
 probabilistic rule (on the right) that is defined for each block.
} \label{matb}
\end{figure}

\begin{figure}[h]
\centering
\includegraphics[width=150mm]{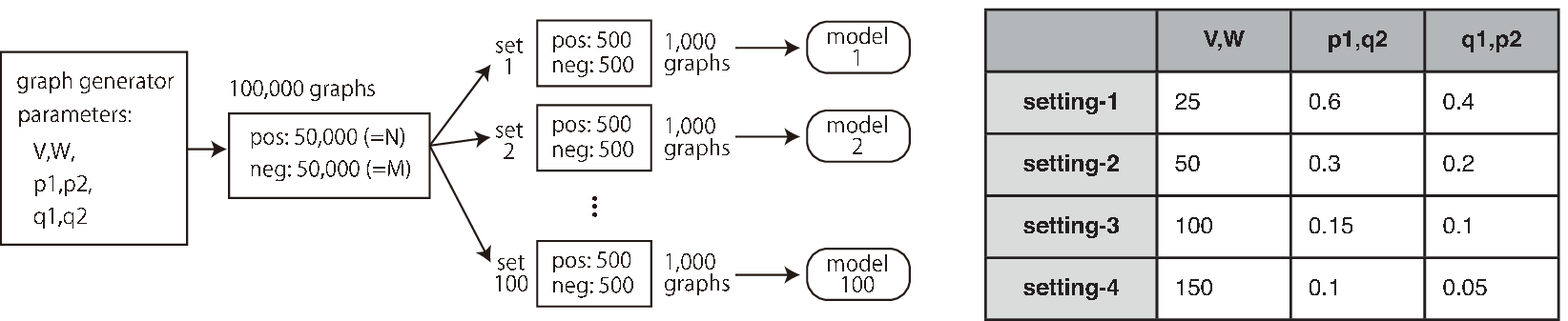}
\caption{The scheme and settings of data generation for numerical
 evaluations. We generated 100 sets of 1,000 random graphs
 consisting of 500/500 graphs for (i)/(ii) in Figure \ref{matb},
 respectively, sharing the same generative rule. 
Thus, we can examine the test error of model $i$ using set $j$ ($i \neq
 j$) in addition to the training error.} \label{datasetting}
\end{figure}

\begin{table}[h]
\caption{Statistics on the generated datasets.}\label{stat}
\centering
\begin{tabular}[t]{lccccccc}
\hline
data set & \# graphs & \multicolumn{3}{c}{\# edges} &
 \multicolumn{3}{c}{\# nodes} \\
 &  & max & min & avg & max & min & avg\\ \hline
setting-1 & 100,000 & 193 & 49 & 119.54 & 162 & 44 & 101.53 \\
setting-2 & 100,000 & 219 & 32 & 120.28 & 182 & 28 & 103.03 \\
setting-3 & 100,000 & 246 & 33 & 115.91 & 202 & 29 & 100.78 \\
setting-4 & 100,000 & 221 & 25 & 104.46 & 187 & 23 & 90.61 \\ \hline
\end{tabular}
\end{table}

\IncMargin{1em}
\begin{algorithm}
\hrule
\vspace*{0.5ex}
\SetKwInput{KwAlgo}{Algorithm}
\KwAlgo{}
\vspace*{0.5ex}
\hrule
\SetKwFor{RepeatTimes}{repeat}{do}{end}
Generate $V+W$ Poisson random numbers $a_1,a_2,\dots,a_{V+W}$ (resample
 if $a_i < 2$)\;
Initialize dictionary $d$\;
Initialize seed-graph pool $X \leftarrow \{\}$\;
\For{$i=1,2,\dots,V+W$}{
 Generate a graph $g$ with two node and an edge between them. (node and edge
 are labeled at random)\;
 \RepeatTimes{$a_i$ times}{
 Select a node $\alpha$ from $g$ at random\;
 Select one of the two choice at random (to ensure the connectivity of $g$)\;
 \Indp
 i) Add an edge from $\alpha$ to a node of $g$ that is not adjacent to
 $\alpha$\;
 ii) Add a new node $\alpha'$ to $g$ and an edge between $\alpha$ and $\alpha'$\;
 \Indm
 }
 \uIf{$\mathtt{MinDFSCode}(g)$ in $d$} 
 {continue\;}
 \Else{
 Register $\mathtt{MinDFSCode}(g)$ to $d$\;
 $X \leftarrow X \cup \{g\}$\;
 }
 }
 Set the first $V$ graphs in $X$ as $A$, and the last $W$ graphs in $X$ as $B$\;
 \RepeatTimes{$N$ times to generate $N$ positive samples}{
 Select $S_V \subseteq A$ with $P(x_i \in S_V)=p_1$ for
 $x_i \in A,i=1,2,\dots,V$\; 
 Select $S_W \subseteq B$ with $P(x_i \in S_W)=q_1$ for $x_i\in B,i=1,2,\dots,W$\; 
 $g \leftarrow$\texttt{combine}($S_V \cup S_W$)\;
 Add $g$ to a set of positive examples\;
 }
 \RepeatTimes{$M$ times to generate $M$ negative samples}{
 Select $S_V \subseteq A$ with $P(x_i \in S_V)=p_2$ for
 $x_i \in A,i=1,2,\dots,V$\; 
 Select $S_W \subseteq B$ with $P(x_i \in S_W)=q_2$ for $x_i\in B,i=1,2,\dots,W$\; 
 $g \leftarrow$\texttt{combine}($S_V \cup S_W$)\;
 Add $g$ to a set of negative examples\;
 }
\vspace*{0.5ex}
\hrule
\caption{The procedure for class-labeled random graph generation.
}\label{randgen}
\end{algorithm}\DecMargin{1em}

\begin{algorithm}\DecMargin{1em}
\hrule
\vspace*{0.5ex}
\SetKwInput{KwAlgo}{Function}
\KwAlgo{\texttt{combine}($x_1,x_2,x_3,\dots$)}
\vspace*{0.5ex}
\hrule
\vspace*{0.5ex}
 $g' \leftarrow x_1$\;
 \ForEach{$g_i$ in $x_2,x_3,\dots$}{
   Select a node of $g'$ and a node of $g_i$ at random\;
   Add an edge between them, and replace $g'$ by the combined graph\;
 }
 \Return{$g'$}
\vspace*{0.5ex}
\hrule
\caption{The algorithm for combining several graphs into a single
 connected graph.}\label{combine}
\end{algorithm}\DecMargin{1em}

\subsection{Evaluating Learning Curves}

We first investigate the convergence property by the learning curves of
the three methods on the same dataset.
Figure~\ref{datasetting} shows a schematic manner of generating an
evaluation dataset by using the random graph generator in Figure \ref{randgen}.
First we generate 100,000 graphs with the fixed parameters of
$p_1,p_2,q_1$, and $q_2$, 
and then divide them into 100 sets, each containing 1,000 graphs (500
positives and 500 negatives). 
Out of each set of 1,000 graphs, we train the model by
either of Adaboost, LPBoost, and 1-norm penalized logistic regression
(denoted by L1-LogReg hereafter). 
We estimate the expected \textit{training error}
by computing each
training error of model $i$ with data set $i$ that is used
to train the model, and averaging over those 100 values obtained from
100 sets. 
Moreover, since all 100 sets share the same probabilistic rule behind
its generation, we can also estimate the expected \textit{test error} by
first randomly choosing 100 pairs of set $i$ and model $j$ ($i \neq j$),
and computing the test error of the model $i$ with data set $j$ that is
not used to train the model, and averaging over those 100 values
obtained from 100 pairs. We use the same fixed 100 
pairs for evaluating all three methods of Adaboost, LPBoost, and
L1-LogReg. Since we have the intermediate model at each iteration, we
can obtain the training and test error at each iteration, which gives us
the learning curves of the three methods regarding the training error as
well as the test error. We use four different settings for control
parameters, which are shown in Figure \ref{datasetting}.
Table \ref{stat} shows the statistics on the datasets generated by these
four settings.
It should be noted that when we draw an averaged learning curve, we need
to perform $100$ trainings and $100$ testings for $1,000$ graphs
\textit{at each iteration}. 
For example, the learning curves of $150$ iterations requires $150 \times
100$ trainings and $150 \times 100$ testings for $1,000$ graphs.

First, for each setting of the data generation in
Figure~\ref{datasetting}, the estimated learning curves are shown in
Figures \ref{result1_1} (setting-1), \ref{result1_2} (setting-2),
\ref{result1_3} (setting-3), and \ref{result1_4} (setting-4). On the
left three panels in each figure, we show 100 learning curves from each
of 100 sets, and their averaged curves (in red color). On the right-most
panel, we show only three averaged curves for comparison.

From these figures, we can first see that the convergence rate was clearly
improved by the proposed algorithm compared to Adaboost and
LPBoost. 
Also we can see that LPBoost accelerated the convergence rate of
Adaboost by totally corrective updates. The convergence behavior of
LPBoost was unstable at the beginning of iterations, which was already
pointed out in the literature 
\citep{Warmuth:2007,Warmuth:2008}, whereas Adaboost and the proposed
method were more stable. LPBoost however often achieved slightly higher
accuracy than L1-LogReg and Adaboost, implying that the hinge loss
function (LPBoost) fits better to the task compared to the logistic
loss (L1-LogReg) or the exponential loss (Adaboost). 

Second, Figures \ref{result2} (setting-1), \ref{result3} (setting-2),
\ref{result4} (setting-3), and \ref{result5} (setting-4) show the
estimated learning curves (averaged over 100 sets) for different
parameters of each learning model. 
Adaboost has only one parameter for the number of iterations, and the
difference in settings did not affect the learning curves, while LPBoost
and L1-LogReg have a parameter for regularization. The parameter of
L1-LogReg affects the training error and test error at convergence, but
in any cases, the result showed more stable and faster 
convergence than the other two methods. Also we could see that
the parameter of LPBoost also affected the instability at the beginning of
iterations. Regarding performance, LPBoost was in many cases slightly
better than the other two, but the error rate at the best parameter
tuning was almost the same in practice for this task.

\subsection{Evaluating the difference in selected subgraph features}\label{featcomp}

Third, we compared the resultant selected subgraph features at
convergence of the three models. To make this evaluation as fair as
possible, we used setting-2 with the parameter of 325 for Adaboost, 0.335
for LPBoost, and 0.008 for L1-LogReg, which were all carefully adjusted
so as to have almost the same number of non-redundant selected features. 
As shown in Table \ref{table1}, we had about
240 features on average, and we also observed that the three methods gave
mostly the same test error of around 0.17. We could see that the original
dataset of setting-2 was generated by combining 100 small graphs in the
seed-graph pool ($V+W=100$), but the number of learned features in Table
\ref{table1} exceeded the number of features to generate the data. This
is because the presence of any graph can be confirmed by the presence of
smaller co-occurring subgraphs that would be contained in unseen samples
with high probability. For the setting of 
Table \ref{table1}, Figure \ref{result9} shows the number of features
with nonzero coefficients at each iteration (averaged over 100 data
sets), all of which converged to around 240 features. 
The number of features of L1-LogReg
increased quickly at the beginning of iterations and then slightly dropped
to the final number. In addition, the variance of the number was larger
than Adaboost and LPBoost. On the other hand, the number of features of
Adaboost and LPBoost increased almost linearly. Note that we count
only non-redundant features by unifying the identical features that are
added at different iterations. We also ignore features with zero
coefficients.

Next, we checked the size of individual subgraph features, which is the
number of edges of each subgraph feature. In Figure \ref{result6}, we
show the size distribution of subgraph features in the seed-graph pool
in the left-most panel that is used to generate the data, and those of
Adaboost, LPBoost, and L1-LogReg to the right. Interestingly, even
though the number of features ($\approx 240$) and the performance of
the three methods ($\approx 0.17$) were almost similar under the setting
of Table \ref{table1}, the selected set of subgraph features was quite 
different. First we can conclude that these learning algorithms do not directly
choose the discriminative subgraph features. Compared to the original
subgraphs stored in the seed-graph pool (up to size 7), all three
methods chose much smaller subgraphs and tried to represent the data by
combining those small subgraphs. In particular, Adaboost and LPBoost
focused on selecting the subgraph features, where their size was less than
and equal to three, mostly the subgraph feature of size two (graphs with
three nodes and two edges). By contrast, L1-LogReg had more balanced size
distribution, and the most frequent subgraph features were of size
three. Figure \ref{result7} shows the number of overlapped subgraph
features between different methods (averaged over 100
sets). Figure~\ref{result6} might give a misleading impression that the
selected features of Adaboost and LPBoost would be similar when the
obtained number of features is similar. Yet they are remarkably
different as we see in Figure \ref{result7}. By looking more closely, 
LPBoost and Adaboost shared a much larger number of features than those
between either of them and L1-LogReg. 
For example, out of 240, 71.8 + 64.42 = 136.22 features are shared
between Adaboost and LPBoost, whereas 91.71 and 99.15 features 
are shared between L1-LogReg and Adaboost, and between L1-LogReg and
LPBoost, respectively.
This implies that L1-LogReg captured a more unique and different set of
subgraph features compared to the other two, as already implied by the
size distribution of Figure \ref{result6}, while as shown in Figure
\ref{result7}, the overlap between the subgraph features obtained by
L1-LogReg and the seed-graph pool was larger than that by Adaboost and
LPBoost.
This suggests that L1-LogReg optimized by our method captured subgraph
features which are relevant to given training data, which would be more
advantageous when we try to interpret the given data. 

\subsection{Evaluating the difference in search space size}

Lastly, Figure \ref{result8} shows the size of search space in the three
methods, i.e. the number of visited nodes in the enumeration tree,
$\mathcal{T}_{\mathrm{visited}}(\mathcal{G}_n)$. Adaboost and LPBoost
start with relatively shallower traversal at the beginning of
iterations, and then search deeper as the iteration proceeds. On the
other hand, the proposed method starts with deep search, and then updates
only a few necessary coefficients by further narrowing search at the
subsequent iterations. Thus, intensive search of necessary features was
done within the beginning of iterations, and the the number of nonzero
coefficients that need to be updated was small at the later iterations. 

\clearpage

\begin{figure}[t]
\centering
\includegraphics[width=\textwidth]{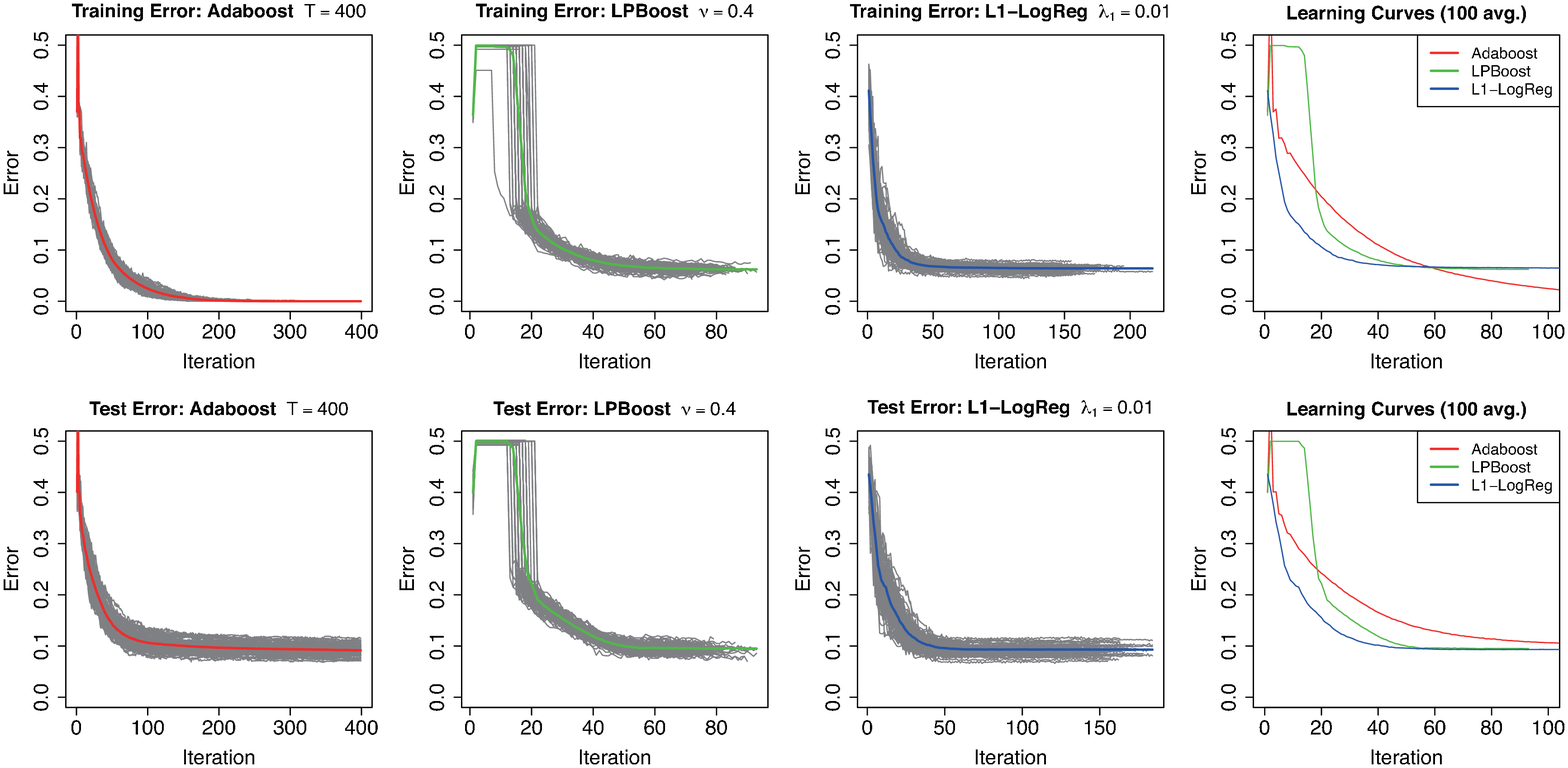}
\caption{Learning curves for setting-1 (average over 100 trials).} \label{result1_1}
\end{figure}
\begin{figure}[t]
\centering
\includegraphics[width=\textwidth]{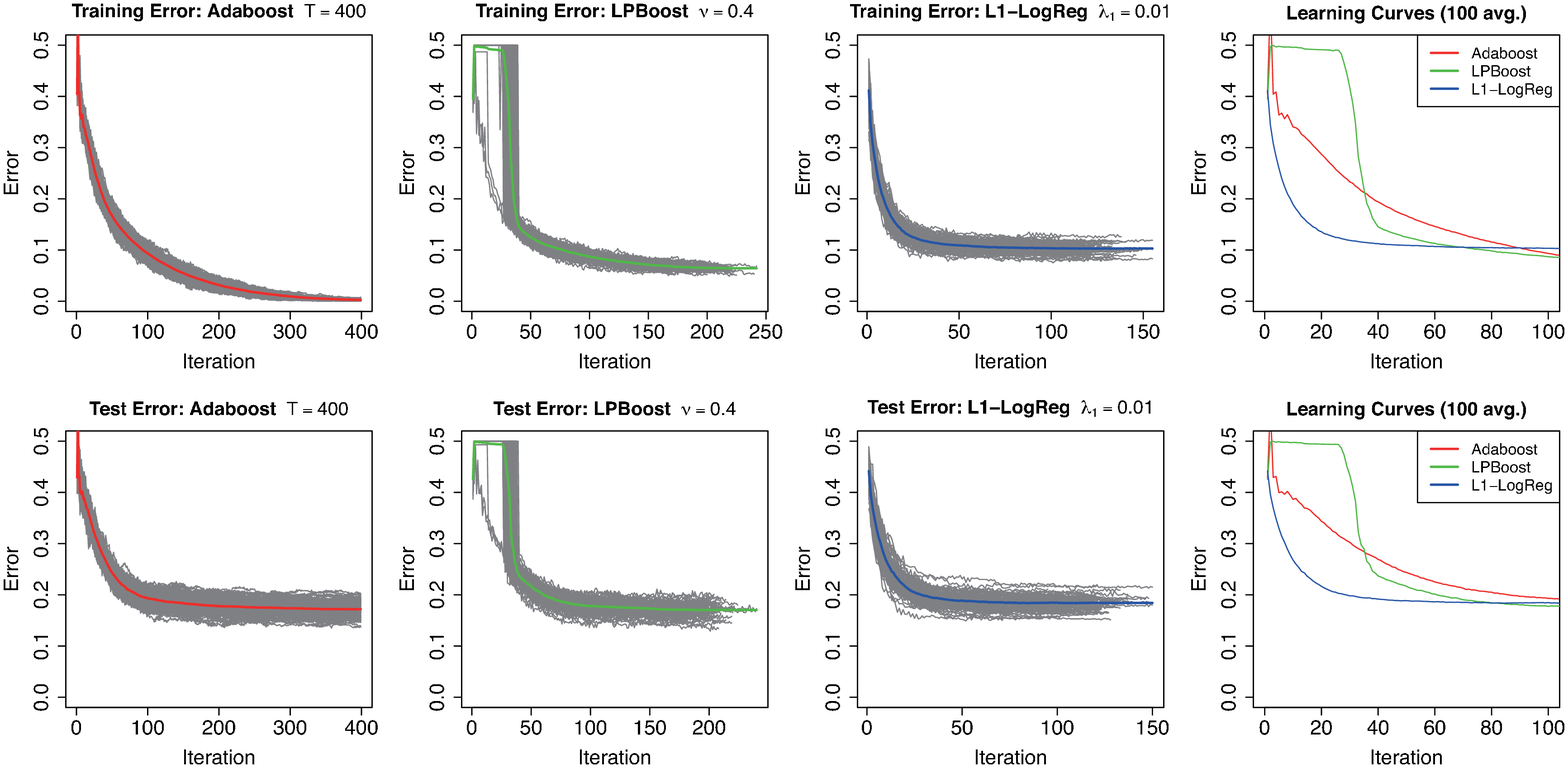}
\caption{Learning curves for setting-2 (average over 100 trials).} \label{result1_2}
\end{figure}
\begin{figure}[t]
\centering
\includegraphics[width=\textwidth]{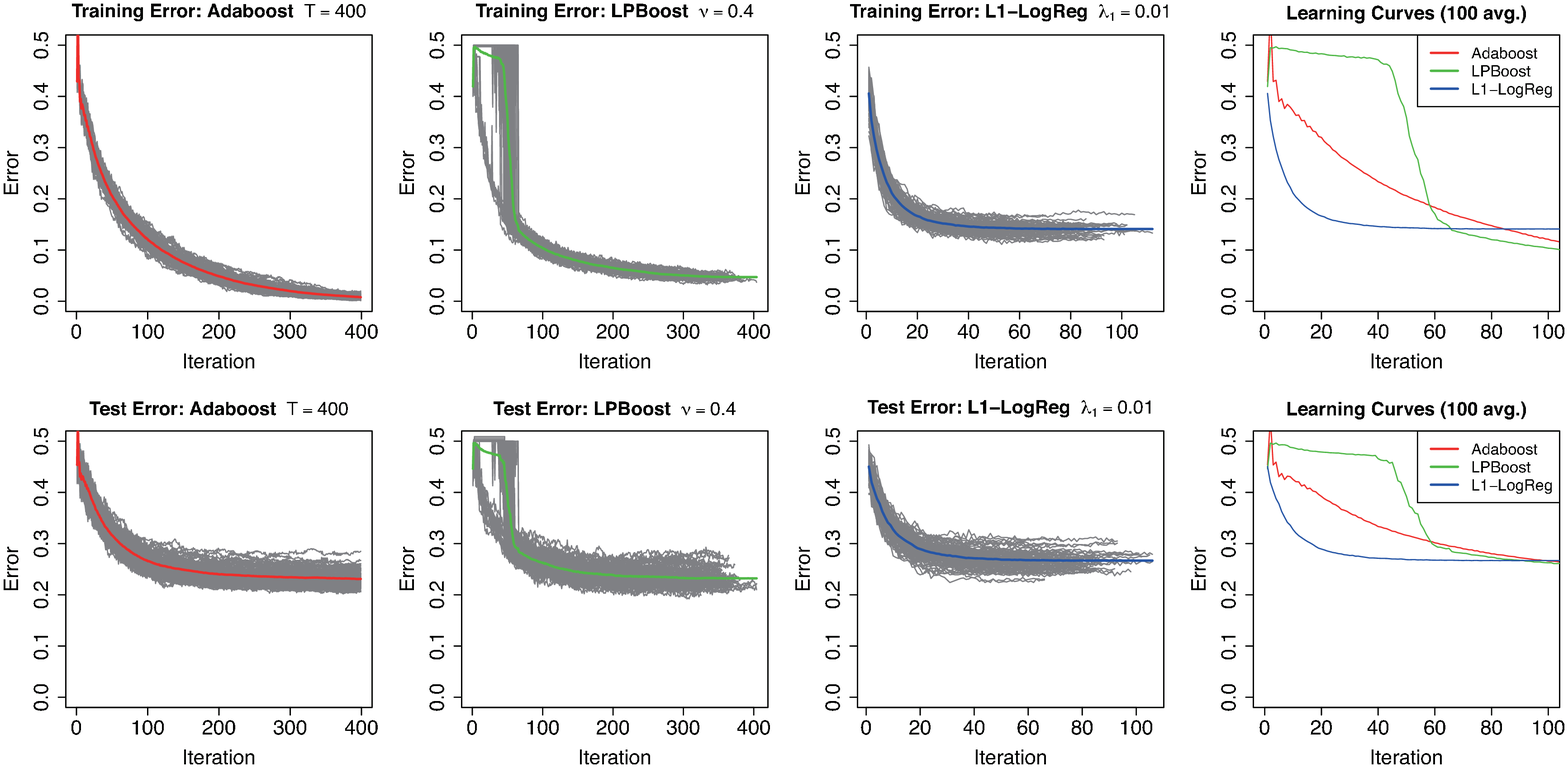}
\caption{Learning curves for setting-3 (average over 100 trials).} \label{result1_3}
\end{figure}
\begin{figure}[t]
\centering
\includegraphics[width=\textwidth]{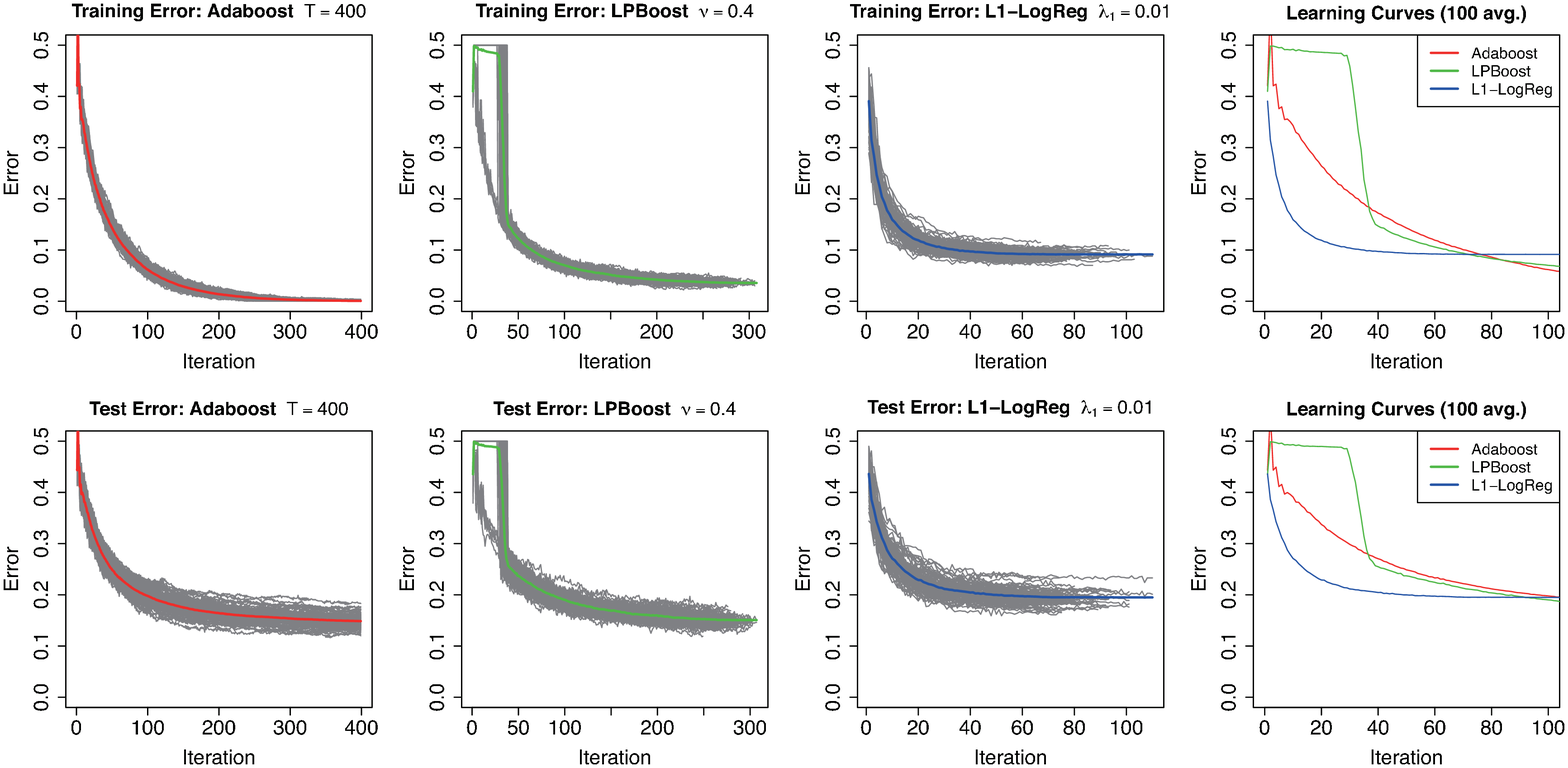}
\caption{Learning curves for setting-4 (average over 100 trials)).} \label{result1_4}
\end{figure}

\clearpage

\begin{figure}[h]
\centering
\vspace*{-2em}
\includegraphics[height=0.42\textheight]{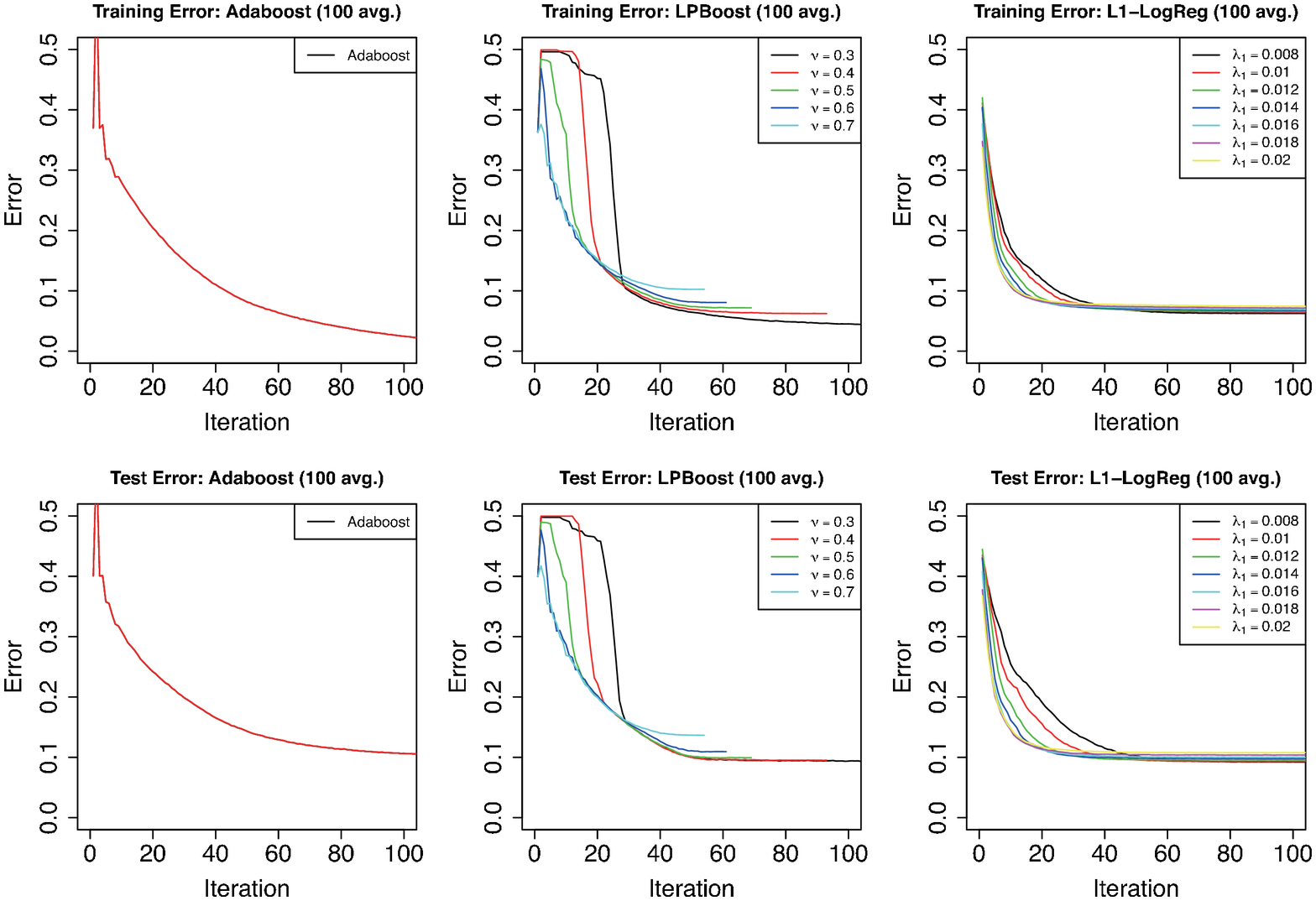}
\caption{Average learning curves for different parameters
 (setting-1).} \label{result2}
\end{figure}
\begin{figure}[h]
\centering
\includegraphics[height=0.42\textheight]{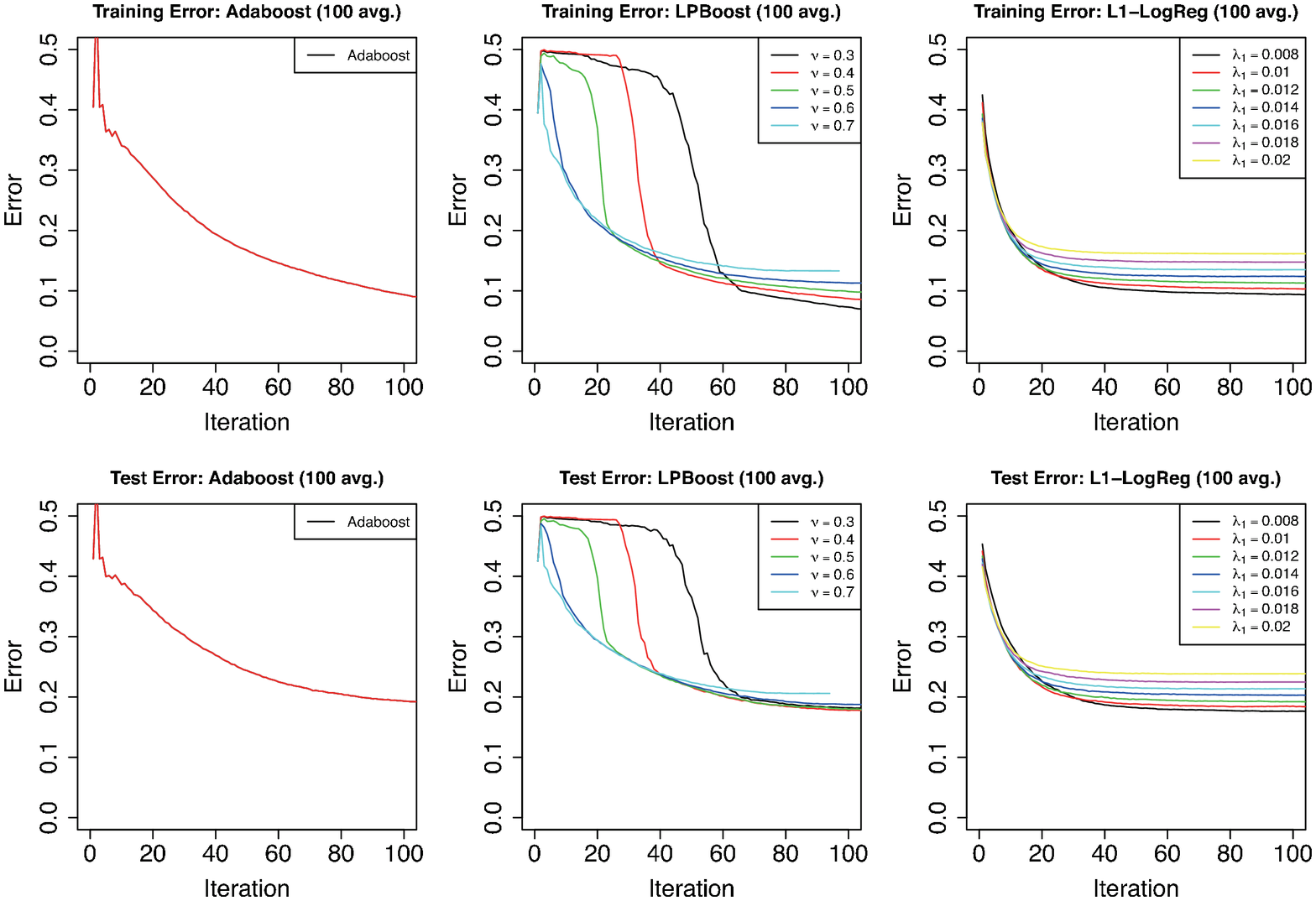}
\caption{Average learning curves for different parameters
 (setting-2).} \label{result3}
\end{figure}

\begin{figure}[h]
\vspace*{-2em}
\centering
\includegraphics[height=0.42\textheight]{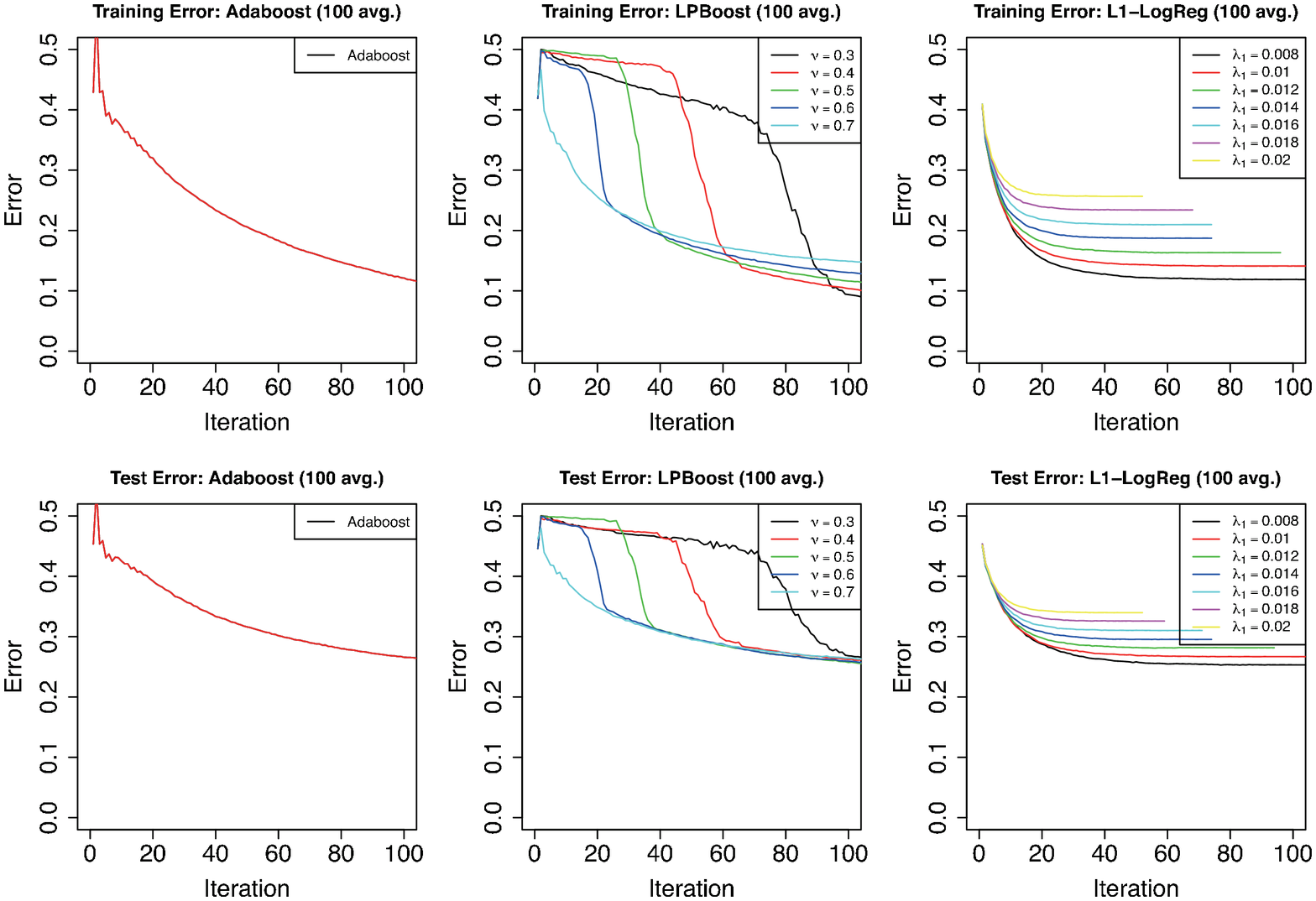}
\caption{Average learning curves for different parameters
 (setting-3).} \label{result4}
\end{figure}
\begin{figure}[h]
\centering
\includegraphics[height=0.42\textheight]{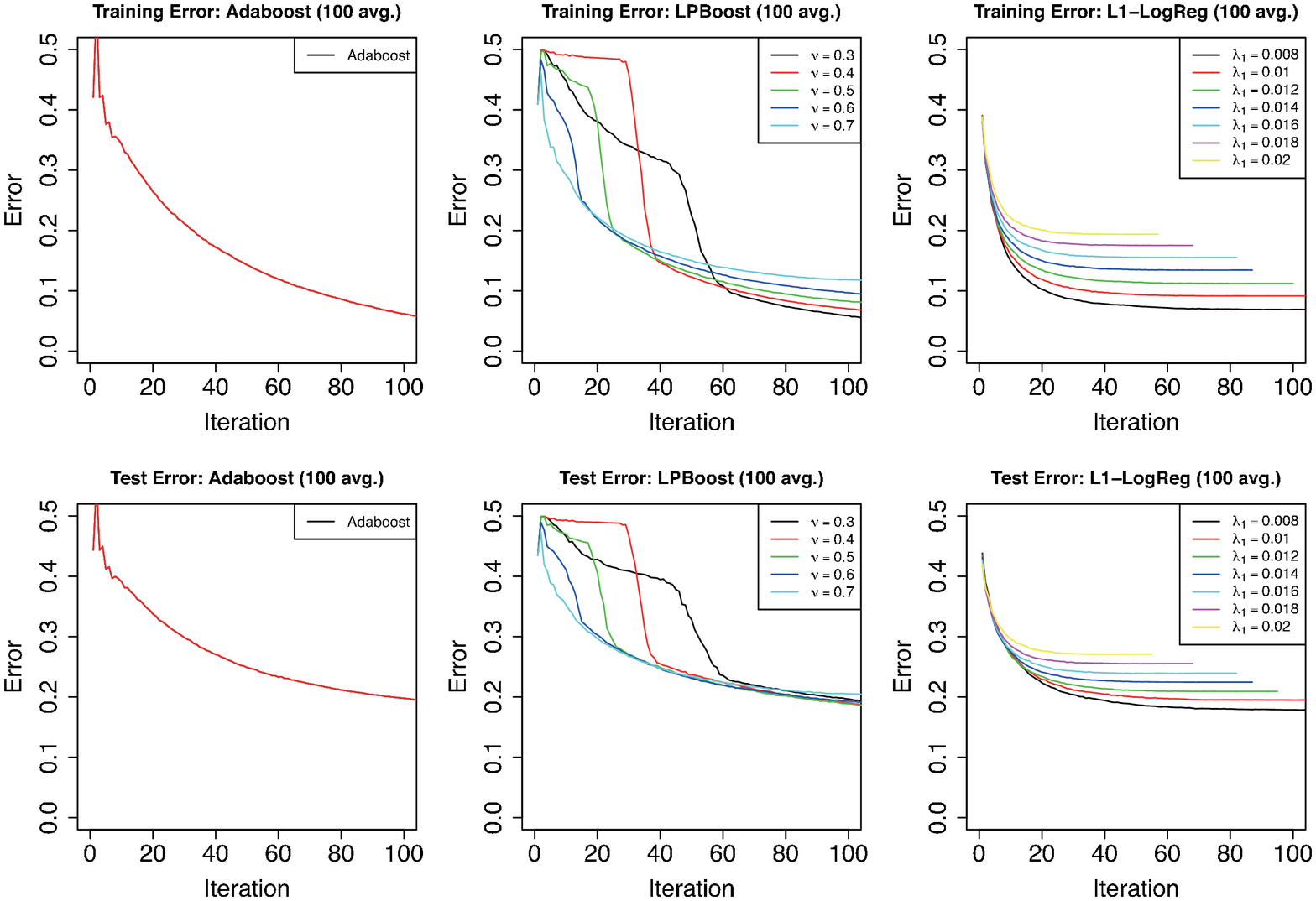}
\caption{Average learning curves for different parameters
 (setting-4).} \label{result5}
\end{figure}

\clearpage

\begin{table}[h]
\caption{Statistics for setting-2 at convergence (average over 100 trials).}\label{table1}
\centering
\begin{tabular}[t]{llllll}
\hline
method & param & \# feat & \# iter & \multicolumn{2}{c}{error} \\
 &  & &  & \multicolumn{1}{c}{(train)} & \multicolumn{1}{c}{(test)}\\ \hline

Adaboost & 325 & 239.94 $\pm$ 8.50  & 325 & 0.0068 & 0.1736 \\
LPBoost & 0.335 &  239.69 $\pm$ 21.80 & 275.91 & 0.0405 &  0.1704 \\
L1-LogReg & 0.008 & 239.36 $\pm$ 29.16  & 122.52 & 0.0931 & 0.1758 \\
 \hline
 &  & $\approx$ 240 & &  & $\approx$ 0.17
\end{tabular}
\end{table}

\begin{figure}[h]
\centering
\includegraphics[width=\textwidth]{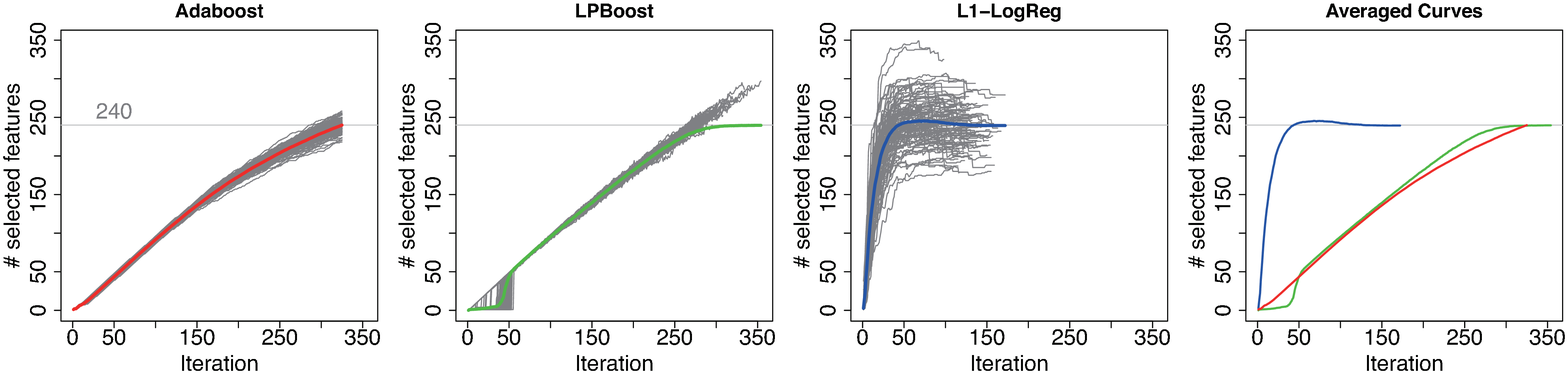}
\caption{The number of unique subgraph features along the iterations.} \label{result9}
\end{figure}

\begin{figure}[h]
\centering
\includegraphics[width=\textwidth]{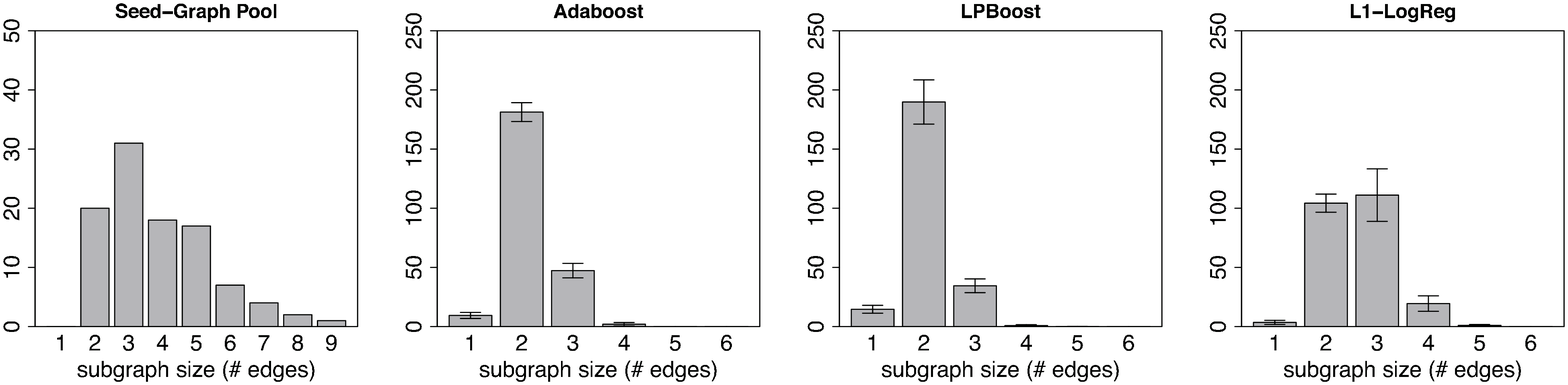}
\caption{Distribution of the size of non-redundant subgraph features (average over 100 trials).} \label{result6}
\end{figure}

\begin{figure}[h]
\centering
\includegraphics[width=120mm]{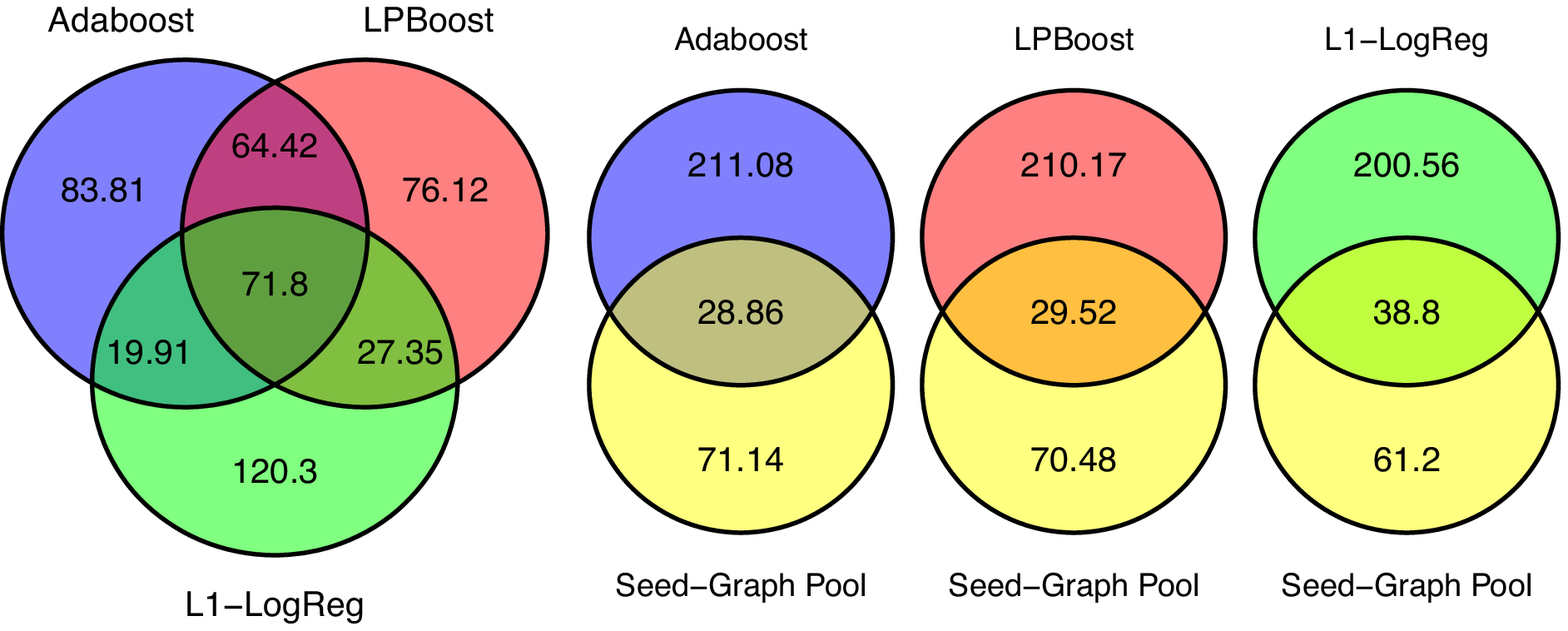}
\caption{The number of isomorphic subgraph features (average over 100 trials).} \label{result7}
\end{figure}

\begin{figure}[h]
\centering
\includegraphics[width=\textwidth]{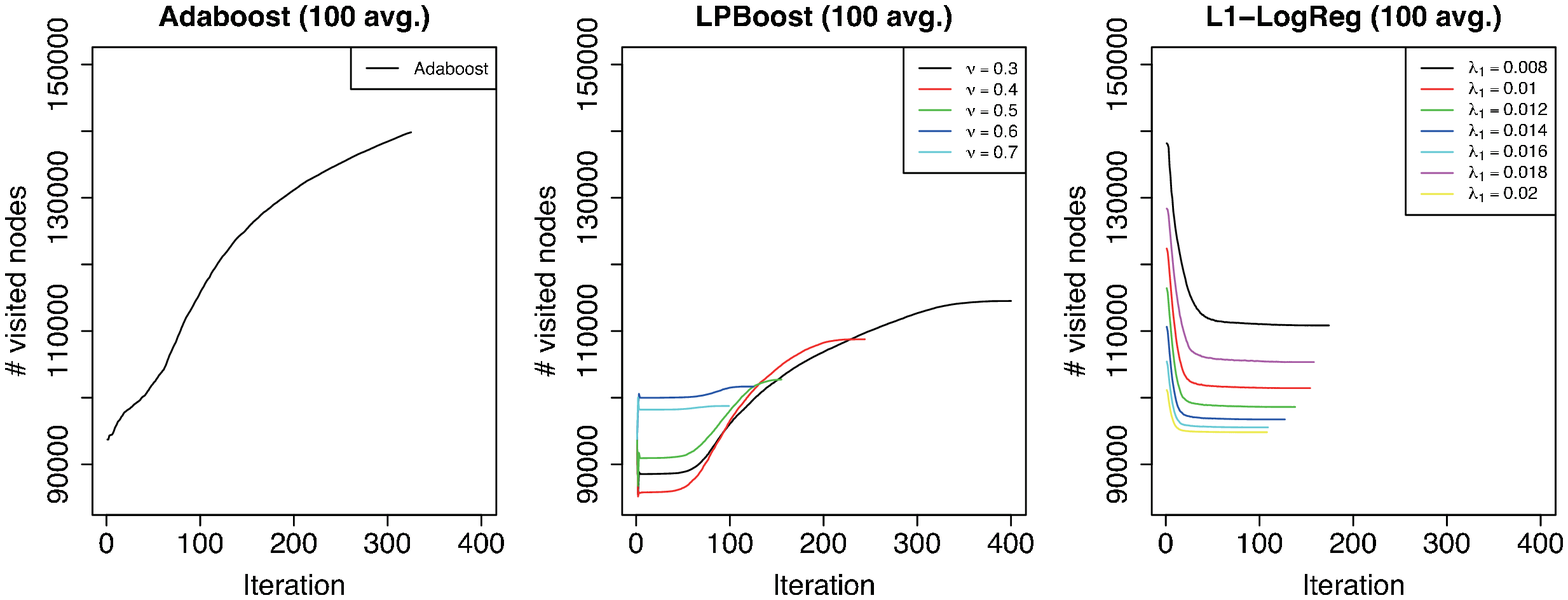}
\caption{The number of visited nodes in the enumeration tree.} \label{result8}
\end{figure}

\begin{figure}[h]
\centering
\includegraphics[width=120mm]{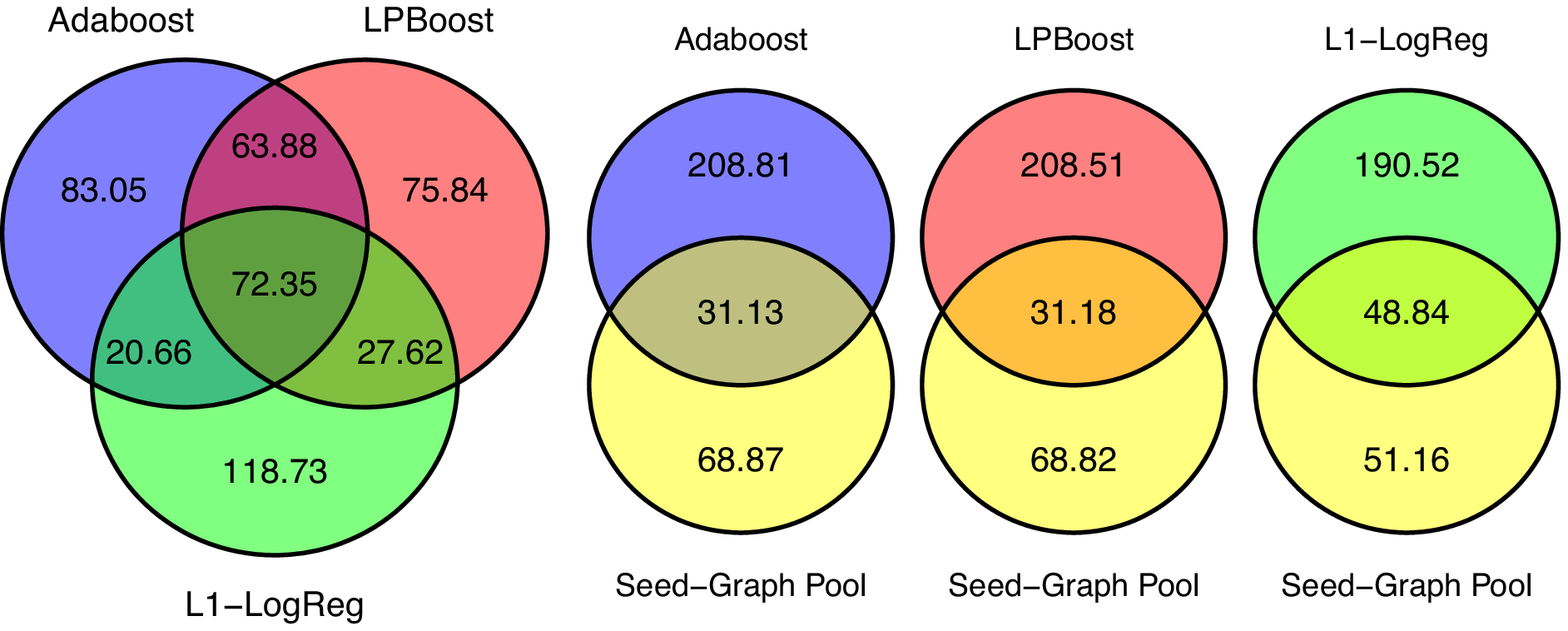}
\caption{The number of overlapping subgraph features by the
 equivalence classes (average over 100 trials).} \label{result10}
\end{figure}

\clearpage

\begin{table}[h]
\caption{Classification accuracy for the CPDB dataset (10-fold CV).}\label{real1}
\centering
\begin{tabular}[t]{cccccr}
\hline
method & & param & \multicolumn{2}{c}{ACC} & \multicolumn{1}{c}{\# feat} \\ 
   & &  & (train) & (test) &  \\ \hline
L1-LogReg& & 0.005 & 0.813 & 0.774 & 80.3 \\
Adaboost& & 500 & 0.945 & 0.772 & 180.3 \\ 
LPBoost && 0.4 & 0.895 & 0.784 & 101.1 \\ \hline
glmnet & FP2 & 0.02 & 0.828 & 0.739 & 1024 \\
(L1-LogReg) & FP3 & 0.03 & 0.676 & 0.628 & 64\\
  & FP4 & 0.02 & 0.785 & 0.721 & 512 \\
 & MACCS & 0.01 & 0.839 & 0.771 & 256 \\ \hline
\end{tabular}
\end{table}

\begin{table}[h]
\caption{Performance and search-space size for the CPDB dataset (10-fold CV).}\label{real2}
\centering
\footnotesize
\begin{tabular}[t]{lcccrrrrrcc}
\hline
method & param & \multicolumn{2}{c}{ACC} & \multicolumn{1}{c}{\# feat} & \multicolumn{1}{c}{\# iter} & \multicolumn{1}{c}{time} & 
\multicolumn{2}{c}{\# visited} & \multicolumn{2}{c}{\% skipped} \\
 &  & (train) & (test) &  & & \multicolumn{1}{c}{(sec)} & \multicolumn{1}{c}{avg} & \multicolumn{1}{c}{max} & \multicolumn{1}{c}{avg} & \multicolumn{1}{c}{max} \\ \hline
L1-LogReg & 0.004 & 0.825 & 0.755 & 95.0 & 57.5 & 8918.94 & 54250.2 & 415079.5 & 96.17\% & 96.22\% \\
          & 0.005 & 0.813 & \textbf{0.774} & 80.3 & 62.3 & 2078.88 & 31472.4 & 128906.7 & 90.40\% & 90.57\%\\
          & 0.006 & 0.803 & 0.762 & 67.6 & 66.6 & 1012.93 & 23354.2 & 85428.7  & 85.74\% & 85.88\% \\
          & 0.008 & 0.779 & 0.750 & 50.9 & 50.6 &  277.69 & 16587.7 & 46213.1  & 76.78\% & 76.82\% \\
          & 0.010 & 0.756 & 0.733 & 39.1 & 46.4 &   97.29 & 12524.3 &
				 26371.4  & 69.95\% & 70.02\% \\ \hline
Adaboost & 600 & 0.950 & 0.769 & 190.0 & 600.0 & 40.67 & 67907.2 & 108425.0 &  &  \\
         & 500 & 0.945 & \textbf{0.772} & 180.3 & 500.0 & 36.21 & 61643.9 & 107421.7 &   &  \\
         & 400 & 0.938 & 0.769 & 168.8 & 400.0 & 29.38 & 53657.1 &
				 96412.1 &   &  \\ \cline{1-9}
LPBoost & 0.3 & 0.939 & 0.748 & 141.4 & 185.9 & 16.02 & 19977.8 & 41921.5 &   &  \\
         & 0.4 & 0.895 & \textbf{0.784} & 101.1 & 126.4 &  7.76 & 15431.7 & 28905.9 &   &  \\
         & 0.5 & 0.858 & 0.767 &  67.1 &  80.2 &  4.39 & 12242.6 &
				 22290.1 &   &  \\ \cline{1-9}
\end{tabular}
\end{table}

\clearpage

\section{Discussion}

In this section, we point out and discuss some new issues about
learning from graphs that we have noticed during developing the proposed
framework. 

\subsection{Bias towards selecting small subgraph features}

As we see in Figure \ref{result9}, all three models of Adaboost,
LPBoost, and L1-LogReg tend to choose smaller subgraph features.
The generated graphs by the algorithm of Figure \ref{randgen} are based
on combining the graphs in a random seed-graph pool. The size of graphs
in the seed-graph 
pool follows a Poisson distribution, which is shown in the histogram
in Figure \ref{result9}, which shows the distributions of subgraphs
captured by the three methods as well.
The distributions by the three methods show that small subgraphs were
captured, because large subgraph features can be identified by checking
multiple smaller subgraphs, and moreover small subgraphs would
occur in the data with higher probability than larger ones. 
Yet there is still remarkable difference in characteristics
between the subgraph features of Adaboost, LPBoost, and
L1-LogReg even when the performance are similar as in Table
\ref{table1}. 
That is, a closer look shows that Adaboost and LPBoost choose subgraphs
of size two rather than size three, and L1-LogReg choose subgraphs of
size both two and three.
This means that comparing to a random seed-graph pool, L1-LogReg comes
with more balanced size of subgraph features, whereas Adaboost and
LPBoost mostly choose the subgraphs of size two only. 
Also note that the features of Adaboost and LPBoost are similar in size
(Figure \ref{result6}), but they are still very different (Figure
\ref{result7}). This finding would suggest that we need to be careful
about the result when we interpret the learned subgraphs in real
applications. 

\subsection{High correlation and equivalence class by perfect
  multicollinearity of subgraph indicators}

The design matrix of Figure \ref{fig:feat_table} may have high
correlations between column vectors (or regressors).  When subgraph
features $x_i$ and $x_j$ are very similar, (for example, only one or two
edges are different between them), the corresponding subgraph indicators
$I(x_i\subseteq g)$ and $I(x_j \subseteq g)$ would take very similar
values. Moreover, the number of samples are generally far less than the
number of possible subgraph features. Therefore, we can have many
exactly identical column vectors that correspond to different subgraph
features, which causes \textit{perfect multicollinearity}.

In most practical cases, we have a particular set of subgraph features,
say $X$, such that any $x \in X$ has the same Boolean vector of
$I_{\mathcal{G}_n}(x)$. 
These subgraph features form an equivalence class
\[
 [x] := \{x'\in \mathcal{X}(\mathcal{G}_n) \mid 
I_{\mathcal{G}_n}(x') = I_{\mathcal{G}_n}(x)
\},
\]
where any
representative subgraph feature $x$ has the same $I_{\mathcal{G}_n}(x)$.
Note that two graphs with very different structures can be in the same
equivalence class if they perfectly co-occur (For example,
disconnected-subgraph patterns).

Therefore, for given $x_i, x_j \in [x]$, we cannot distinguish if either
of two subgraph indicators is better than the other, just by using
$\mathcal{G}_n$. 
A heuristics is that the smallest subgraph in $[x]$ might be
good as the representative subgraph, because smaller subgraphs are expected
to occur in unseen graphs with a higher probability than that of larger ones.
This point should be carefully treated when we are interested in
interpreting the selected set of subgraph features for the application
purpose. We can always have many other candidates in the equivalence
class that has exactly the same effect as a subgraph feature. 

The existing methods such as Adaboost and LPBoost add a single \text{best}
subgraph feature at each iteration with branch and bound, and thus
usually ignore this problem, because they just take the first found
best subgraph. 
On the other hand, the proposed method allows to add multiple subgraph
features at each iteration. Hence we need to handle this point
explicitly. In this paper, for easy comparisons, we just take the first
found subgraph $x$ (following the other methods) and ignore the
subsequently found subgraph in the same equivalence class $[x]$ by a
hashing based on $I_{\mathcal{G}_n}(x)$. This is simply because in this
paper we are
interested in the number of unique subgraphs. However we can
output the complete set of the equivalence classes and analyze them when we
are interested in interpreting the selected subgraphs. In addition, when we
are interested in prediction performance only, we can instead use 2-norm
regularization in our framework of Problem \ref{prob1} to average over
highly correlated features, although this method may increase the number
of selected subgraph features.

Lastly, it should be noted that this equivalence class problem, being
caused by perfect multicollinearity, might cast a question regarding the
validity of comparing feature sets by graph isomorphism as was done in
Section \ref{featcomp} and Figure \ref{result7}.  
We conducted the same experiments by replacing
the isomorphism test with an equivalence-class test based on the
identity of $I_{\mathcal{G}_n}(x)$. The result is shown
in Figure \ref{result10} and, from the comparison between Figures
\ref{result7} and \ref{result10}, we can confirm that our discussion in
Section \ref{featcomp} would be still valid even when we take the
equivalence class into account. 

\subsection{Performance on real dataset: The CPDB dataset}

Despite the faster convergence in theory, our straightforward
implementation of the algorithm of Figure \ref{algo} can be slow
or even hard to compute for real-world datasets, at least in the current 
form. Thus future work includes improving the practical running time by
pursuing more efficient algorithms and data structures.

In order to discuss this point, we show the performance result for a
real dataset. We use the mutagenicity data for carcinogenic potency database
(CPDB) compounds \citep{CPDB:2004}, which consists of 684 graphs
(mutagens: 341, nonmutagens: 343) for a binary classification task. This
dataset has been widely used for benchmarking (See, for example,
\cite{Saigo:2009}). Table \ref{real1} shows the comparison result of the
classification accuracy (ACC) and the number of obtained subgraph
features to attain that accuracy. We included a standard method in
chemoinformatics as a baseline, in which we first computed the
fingerprint bit vector and then applied standard 1-norm penalized
logistic regression optimized by glmnet~\citep{Friedman:2010} to the
fingerprint. For fingerprints, we used 
four different fingerprints, FP1, FP2, FP3, and MACCS, generated by
Open Babel\footnote{Open Babel v2.3.0 documentation: Molecular
fingerprints and similarity searching. \url{http://openbabel.org/docs/dev/Features/Fingerprints.html}}. In terms of classification performance,
L1-LogReg achieved similar accuracy to the best one by Adaboost and
LPBoost with a much smaller number of subgraph features, whereas our
method examined a larger search space than Adaboost and 
LPBoost as we see in Table \ref{real2}. 
Table \ref{real2} shows the detailed statistics behind Table
\ref{real1}, where ``time (sec)'', ``\# visited'' and ``\% skipped''
denote the CPU time\footnote{The CPU time is measured in a workstation
with 2 $\times$ 2.93 GHz 6-Core Intel Xeon CPUs and 64 GB Memory.} in
seconds, the number of visited nodes in the enumeration tree, and the
percentage of skipped nodes due to perfect multicollinearity,
respectively.
Note that this CPU time cannot be a consistent measure, because the
implementations of Adaboost and LPBoost maintain (or \textit{cache}) the 
whole enumeration tree in the main memory (See Section 2.2 in
\cite{Saigo:2009}), whereas we did not use such implementation
techniques. 
In this light the number of visited nodes would be the only measure to
methodologically compare methods in Table \ref{real2} with each other.

When we focus on our method only, from the results of ``time
(sec)'', we can find that we need to decrease the regularization
parameter $\lambda_1$ to a very small value, which would directly loosen
the pruning condition of Theorem \ref{cond2}.
As a result, the CPU time grows exponentially as the size of
search space increases. One interesting point however is that as shown
in ``\% skipped'', the redundant visits constitute a significant
fraction of the total number of visits, which entails an exponentially
long running time.
If we can remove this bottleneck by skipping these unnecessary visits in
some way, the performance can be largely improved. Another interesting
point is that our framework needed the longest time at the first
iteration as shown in Figure \ref{result8}, but after several iterations
the search-space size quickly decreases. If we adopt some warm start
such as a \textit{continuation (homotopy)}
procedure \citep{Yun:2011,Wen:2010,Hale:2008a,Hale:2008b}, 
it would also improve the performance. It would be future work to
provide more realistic implementation by improving these points from 
algorithmic viewpoints.

\subsection{Correctness of implementations: How should we debug them?}

It is generally hard to make sure that a learning algorithm for graphs works
correctly. There is no easy way to visualize a large set of graphs and
subgraph features in a consistent form. This originates from the
difficulty to visually find out graph isomorphism and subgraph isomorphism.
Thus it would be an important problem how to check the correctness of
the implemented algorithm. We know that the complete check is impossible
in practice, and thus here we simply emphasize the importance of careful
double check with systematically generated dataset.

For example, our implementation, fully written in C++ from the scratch, is
based on the gSpan algorithm to implicitly build and traverse an
enumeration tree $\mathcal{T}(\mathcal{G}_n)$.
We ensure that our implementation of gSpan works in the exactly same way to
three other existing implementations with randomly generated graphs
of various types\footnote{We developed this faster implementation in the previous
project. See \cite{Takigawa:2011a} for details.}. Similarly, we
ensure that our implementation of Figure \ref{algo} works as expected
using a simple double check as follows.
When we limit the enumeration tree down to the $i$-th level, it is
relatively easy to enumerate all bounded-size subgraphs in
$\mathcal{X}(\mathcal{G}_n)$ to a certain level of $i$. Given the
results in Figure~\ref{result6}, even if we can check only up to $i=5$ or so, 
it should be valuable to check the identity to the case
where we first generate the data matrix explicitly and apply the block
coordinate gradient descent to the data matrix. Thus, for this
level-limited enumeration tree up to the $i$-th level, we can explicitly
generate the data matrix in the form of Figure
\ref{fig:feat_table}. Applying the block coordinate gradient descent to
this matrix is supposed to produce exactly the same coefficients and
function values at each iteration as those which can be obtained when we
limit $\mathcal{T}(\mathcal{G}_n)$ by size. 
Using the original implementation of block coordinate gradient descent
by \cite{Yun:2011}, we double-check that the nonzero coefficient values
and their indices are exactly the same at each iteration as our
implementation with the level limitation. In addition, we double-check
by applying glmnet \citep{Friedman:2010} to the data matrix and
confirming if the values of the objective function at convergence are
the same (because both glmnet and block coordinate gradient
descent can be applied to the same problem of 1-norm penalized logistic
regression). Note that glmnet in R first automatically standardizes the
input variables, and thus we need to ignore this preprocessing part.

\section{Conclusions}

We have developed a supervised-learning algorithm from graph data---a
set of graphs---for arbitrary twice-differentiable loss functions and
sparse linear models over all possible subgraph features. To date, it
has been shown that we can perform, under all possible subgraph
features, several specific types of sparse learning such as 
Adaboost \citep{Kudo:2005}, LPBoost \citep{Saigo:2009}, LARS/LASSO
\citep{Tsuda:2007}, and sparse PLS 
 regression \citep{Saigo:2008b}. 
We have investigated the underlying idea common to these
 preceding studies, including that we call \textit{Morishita-Kudo
 bounds}, and generalized this idea to a unifying bounding technique for
 arbitrary separable functions. 
Then, combining this generalization with further algorithmic
considerations including \textit{depth-first dictionary passing}, we
have showed that we can solve a wider class of sparse learning
problems by carefully carrying out block coordinate gradient descent
over the tree-shaped search space of infinite subgraph features. We then 
 numerically studied the difference from the existing approaches in a
 variety of viewpoints, including convergence property, selected
 subgraph features, and search-space sizes. 
As a result, from our experiments, we could observe several unnoticed
issues including the bias towards selecting small subgraph features, the
high correlation structures of subgraph indicators, and the perfect
multicollinearity caused by the equivalence classes. We believe 
that our results and findings contribute to the advance of
understanding in the field of general supervised learning from a set of
graphs with infinite subgraph features, and also restimulate this
direction of research. 

\section*{Acknowledgments}

This work was supported in part by JSPS/MEXT KAKENHI Grant Number 23710233 and 24300054. 

\appendix

 \section*{Appendix A1: Proof of Lemma \ref{cond1}}

For applying coordinate descent to \eqref{Ttheta}, we first set only the
$j$-th element of $\theta$ free, i.e. $\theta_j = z$, and fix all the
remaining elements by the current value of $\theta(t)$ as $\theta_k =
\theta(t)_k$ for all $k\neq j$. 
Then, since \eqref{Ttheta} can be written as
\[
\arg\min_{\theta_1,\theta_2,\dots} \biggl[ \,\,
\sum_{k=0}^\infty \nabla f(\theta(t))_k (\theta_k - \theta(t)_k)
+ \frac{1}{2} \sum_{k=0}^\infty \sum_{l=0}^\infty H(t)_{kl}
(\theta_k - \theta(t)_k)(\theta_l - \theta(t)_l) + \sum_{k=1}^\infty
 \lambda_1 |\theta_k| 
\,\, \biggr],
\]
we have the following univariate (one-dimensional) optimization
problem for $z$
\begin{equation}\label{univar1}
  T(\theta(t))_j = \arg \min_{z} 
\biggl[ \,\,
\nabla f(\theta(t))_j (z - \theta(t)_j)
 + \frac{1}{2} H(t)_{jj}
(z -\theta(t)_j)^2 + \lambda_1 |z|
\,\, \biggr].
\end{equation}
On the other hand, for any $a \geqslant 0$, completing the square gives us
\begin{equation}\label{univar2}
\arg\min_{z \in \mathbb{R}} \bigl[\,\, a z^2 + b z + c + \lambda
 |z|\,\,\bigr] =
\begin{cases}
-\frac{1}{2a}(b + \lambda) & b < -\lambda \\
-\frac{1}{2a}(b - \lambda) & b > \lambda \\
0 & -\lambda \leqslant b \leqslant \lambda.
\end{cases}
\end{equation}
Then, by rewriting \eqref{univar1} in the form of \eqref{univar2}, we obtain
the result in the lemma.

 \section*{Appendix A2: Proof of Theorem \ref{cond2}}

Let $L'_\mu(y,\mu) := \partial L(y,\mu)/\partial \mu$.  Then since $\partial \mu(g_i,\theta(t))/\partial \theta(t)_k =
I(x_k \subseteq g_i)$, we have
\begin{equation}\label{a2func}
\sum_{i=1}^n \frac{\partial L(y_i,\mu(g_i;\theta(t)))}{\partial
\theta(t)_k} =
\sum_{i=1}^n L'_{\mu}(y_i,\mu(g_i,\theta(t))) I(x_k \subseteq g_i).
\end{equation}
It will be noticed that this function is separable in terms of values of
$I_{\mathcal{G}_n}(x_k)$ because each $L'_{\mu}(y_i,\mu(g_i,\theta(t)))$
is constant after we substitute the current value for $\theta(t)$. From
Lemma \ref{boolprop}, for an
$n$-dimensional Boolean vector $I_{\mathcal{G}_n}(x)$,
we have $1(I_{\mathcal{G}_n}(x_k)) \subseteq 1(I_{\mathcal{G}_n}(x_j))$
for $x_k \in \mathcal{T}(x_j)$.
Thus we can obtain the Morishita-Kudo bounds defined in Lemma
\ref{mkbounds} for \eqref{a2func} by letting $v := I_{\mathcal{G}_n}(x_k)$ and $u :=
I_{\mathcal{G}_n}(x_j)$ because $1(v)\subseteq 1(u)$. In reality, the Morishita-Kudo
bounds for \eqref{a2func} can be computed as
\begin{align*}
 \underline{L}_j(t) &:=  \sum_{i \in 1(I_{\mathcal{G}_n}(x_j))} \min\bigl(
 L'_\mu(y_i,\mu(g_i;\theta(t))), 0 \bigr),\\
 \overline{L}_j(t) &:=  \sum_{i \in 1(I_{\mathcal{G}_n}(x_j))} \max\bigl(
 L'_\mu(y_i,\mu(g_i;\theta(t))), 0 \bigr),
\end{align*} 
and these two bounds satisfy, for any $k$ such that $x_k \in \mathcal{T}(x_j)$,
\[
 \underline{L}_j(t) \leqslant \sum_{i=1}^n \frac{\partial
 L(y_i,\mu(g_i;\theta(t)))}{\partial 
\theta(t)_k} \leqslant \overline{L}_j(t).
\]
On the other hand, regarding the term $(\lambda_2 -
H(t)_{kk})\theta(t)_k$, since
\[
 (\lambda_2 -
H(t)_{jj})\theta(t)_j = \begin{cases}
			 0 & (\theta(t)_j = 0)\\
			  (\lambda_2 -
H(t)_{jj})\theta(t)_j & (\theta(t)_j \neq 0)
			\end{cases}
\]
and we already have the nonzero elements $\theta(t)_j$ for $j$ such that
$\theta(t)_j \neq 0$, $\underline{B}_j(t)$ and $\overline{B}_j(t)$ are
obtained just as the smallest and largest elements of the finite sets
indexed by $\{k\mid x_k \in \mathcal{T}(x_j), \theta(t)_k \neq 0\}$. Thus we have $\underline{B}_j(t)
\leqslant (\lambda_2 -
H(t)_{kk})\theta(t)_k \leqslant \overline{B}_j(t)$ directly by
\begin{align*}
\underline{B}_j(t) & = \min \left(
0, \min_{l \in \{k\mid x_k \in \mathcal{T}(x_j), \theta(t)_k \neq 0\}}
 (\lambda_2 - H(t)_{ll}) \theta(t)_l
\right),\\
\overline{B}_j(t) &= \max \left(
0, \max_{l \in \{k\mid x_k \in \mathcal{T}(x_j), \theta(t)_k \neq 0\}}
 (\lambda_2 - H(t)_{ll}) \theta(t)_l
\right).
\end{align*}
It should be noted that as we see in Section \ref{sec_dfcp}, the
depth-first dictionary passing is an efficient way to obtain
$\underline{B}_j(t)$ and $\underline{B}_j(t)$ during the traversal of enumeration tree $\mathcal{T}(\mathcal{G}_n)$.

Now from the inequality \eqref{ineq}, we have
\[
 \left|
\sum_{i=1}^n \frac{\partial
 L(y_i,\mu(g_i;\theta(t)))}{\partial 
\theta(t)_k}
+
(\lambda_2 -H(t)_{kk})\theta(t)_k
\right| \leqslant \max\{\overline{L}_j(t)+\overline{B}_j(t),-\underline{L}_j(t)-\underline{B}_j(t)\}.
\]
The statement in Theorem \ref{cond2} follows by combining this observation with Lemma \ref{cond1}.

\bibliographystyle{plainnat}
\bibliography{glearn}

\begin{thebibliography}{53}
\providecommand{\natexlab}[1]{#1}
\providecommand{\url}[1]{\texttt{#1}}
\expandafter\ifx\csname urlstyle\endcsname\relax
  \providecommand{\doi}[1]{doi: #1}\else
  \providecommand{\doi}{doi: \begingroup \urlstyle{rm}\Url}\fi

\bibitem[Beck and Teboulle(2009)]{Beck:2009}
Amir Beck and Marc Teboulle.
\newblock A fast iterative shrinkage-thresholding algorithm for linear inverse
  problems.
\newblock \emph{SIAM Journal on Imaging Sciences}, 2:\penalty0 183--202, 2009.

\bibitem[Borgwardt et~al.(2005)Borgwardt, Ong, Sch\"{o}nauer, Vishwanathan,
  Smola, and Kriegel]{Borgwardt:2005}
Karsten~M. Borgwardt, Cheng~Soon Ong, Stefan Sch\"{o}nauer, S.~V.~N.
  Vishwanathan, Alex~J. Smola, and Hans-Peter Kriegel.
\newblock Protein function prediction via graph kernels.
\newblock \emph{Bioinformatics}, 21\penalty0 (1):\penalty0 i47--i56, 2005.

\bibitem[Combettes and Wajs(2005)]{Combettes:2005}
Patrick~L. Combettes and Val\'{e}rie~R. Wajs.
\newblock {Signal recovery by proximal forward-backward splitting}.
\newblock \emph{Multiscale Modeling and Simulation}, 4\penalty0 (4):\penalty0
  1168--1200, 2005.

\bibitem[Daubechies et~al.(2004)Daubechies, Defrise, and Mol]{Daubechies:2004}
Ingrid Daubechies, Michel Defrise, and Christine~De Mol.
\newblock {An iterative thresholding algorithm for linear inverse problems with
  a sparsity constraint}.
\newblock \emph{Communications on Pure and Applied Mathematics}, 57\penalty0
  (11):\penalty0 1413--1457, 2004.

\bibitem[Deshpande et~al.(2005)Deshpande, Kuramochi, Wale, and
  Karypis]{Dashpande:2005}
Mukund Deshpande, Michihiro Kuramochi, Nikil Wale, and George Karypis.
\newblock Frequent substructure-based approaches for classifying chemical
  compounds.
\newblock \emph{IEEE Transactions on Knowledge and Data Engineering},
  17\penalty0 (8):\penalty0 1036--1050, 2005.

\bibitem[Figueiredo and Nowak(2003)]{Figueiredo:2003}
M{\'a}rio A.~T. Figueiredo and Robert~D. Nowak.
\newblock An {EM} algorithm for wavelet-based image restoration.
\newblock \emph{IEEE Transactions on Image Processing}, 12\penalty0
  (8):\penalty0 906--916, 2003.

\bibitem[Friedman et~al.(2007)Friedman, Hastie, H{\"o}fling, and
  Tibshirani]{Friedman:2007}
Jerome~H. Friedman, Trevor Hastie, Holger H{\"o}fling, and Rob Tibshirani.
\newblock Pathwise coordinate optimization.
\newblock \emph{The Annals of Applied Statistics}, 1\penalty0 (2):\penalty0
  302--332, 2007.

\bibitem[Friedman et~al.(2010)Friedman, Hastie, and Tibshirani]{Friedman:2010}
Jerome~H. Friedman, Trevor Hastie, and Rob Tibshirani.
\newblock Regularization paths for generalized linear models via coordinate
  descent.
\newblock \emph{Journal of Statistical Software}, 33\penalty0 (1):\penalty0
  1--22, 2010.

\bibitem[Fr\"{o}hlich et~al.(2006)Fr\"{o}hlich, Wegner, Sieker, and
  Zell]{Frohlich:2006}
Holger Fr\"{o}hlich, J\"{o}rg~K. Wegner, Florian Sieker, and Andreas Zell.
\newblock Kernel functions for attributed molecular graphs---a new similarity
  based approach to {ADME} prediction in classification and regression.
\newblock \emph{QSAR \& Combinatorial Science}, 25\penalty0 (4):\penalty0
  317--326, 2006.

\bibitem[Fukushima and Mine(1981)]{Fukushima:1981}
Masao Fukushima and Hisashi Mine.
\newblock A generalized proximal point algorithm for certain non-convex
  minimization problems.
\newblock \emph{International Journal of Systems Science}, 12\penalty0
  (8):\penalty0 989--1000, 1981.

\bibitem[G{\"a}rtner et~al.(2003)G{\"a}rtner, Flach, and Wrobel]{Gartner:2003}
Thomas G{\"a}rtner, Peter~A. Flach, and Stefan Wrobel.
\newblock On graph kernels: Hardness results and efficient alternatives.
\newblock In \emph{Proceedings of the 16th Annual Conference on Computational
  Learning Theory (COLT) and 7th Kernel Workshop}, pages 129--143, 2003.

\bibitem[Hale et~al.(2008{\natexlab{a}})Hale, Yin, and Zhang]{Hale:2008a}
Elaine~T. Hale, Wotao Yin, and Yin Zhang.
\newblock Fixed-point continuation method for $\ell_1$-minimization with
  applications to compressed sensing.
\newblock \emph{SIAM Journal on Optimization}, 19:\penalty0 1107--1130,
  2008{\natexlab{a}}.

\bibitem[Hale et~al.(2008{\natexlab{b}})Hale, Yin, and Zhang]{Hale:2008b}
Elaine~T. Hale, Wotao Yin, and Yin Zhang.
\newblock Fixed-point continuation for $\ell_1$-minimization: Methodology and
  convergence.
\newblock \emph{SIAM Journal on Optimization}, 19\penalty0 (3):\penalty0
  1107--1130, 2008{\natexlab{b}}.

\bibitem[Hamada et~al.(2006)Hamada, Tsuda, Kudo, Kin, and Asai]{Hamada:2006}
Michiaki Hamada, Koji Tsuda, Taku Kudo, Taishin Kin, and Kiyoshi Asai.
\newblock Mining frequent stem patterns from unaligned {RNA} sequences.
\newblock \emph{Bioinformatics}, 22\penalty0 (20):\penalty0 2480--2487, 2006.

\bibitem[Harchaoui and Bach(2007)]{Harchaoui:2007}
Za\"{i}d Harchaoui and Francis Bach.
\newblock Image classification with segmentation graph kernels.
\newblock In \emph{Proceedings of the IEEE Computer Society Conference on
  Computer Vision and Pattern Recognition (CVPR)}, pages 1--8, Minneapolis,
  Minnesota, USA, 2007.

\bibitem[Hashimoto et~al.(2008)Hashimoto, Takigawa, Shiga, Kanehisa, and
  Mamitsuka]{Hashimoto:2008}
Kosuke Hashimoto, Ichigaku Takigawa, Motoki Shiga, Minoru Kanehisa, and Hiroshi
  Mamitsuka.
\newblock Mining significant tree patterns in carbohydrate sugar chains.
\newblock \emph{Bioinformatics}, 24\penalty0 (16):\penalty0 i167--i173, 2008.

\bibitem[Haussler(1999)]{Haussler:1999}
David Haussler.
\newblock Convolution kernels on discrete structures.
\newblock Technical Report UCS-CRL-99-10, University of California at Santa
  Cruz, Santa Cruz, California, USA, 1999.

\bibitem[Helma et~al.(2004)Helma, Cramer, Kramer, and Raedt]{CPDB:2004}
Christoph Helma, Tobias Cramer, Stefan Kramer, and Luc~De Raedt.
\newblock Data mining and machine learning techniques for the identification of
  mutagenicity inducing substructures and structure activity relationships of
  noncongeneric compounds.
\newblock \emph{Journal of Chemical Information and Modeling}, 44\penalty0
  (4):\penalty0 1402--1411, 2004.

\bibitem[Karklin et~al.(2005)Karklin, Meraz, and Holbrook]{Karklin:2005}
Yan Karklin, Richard~F. Meraz, and Stephen~R. Holbrook.
\newblock Classification of non-coding {RNA} using graph representations of
  secondary structure.
\newblock In \emph{Proceedings of the Pacific Symposium on Biocomputing (PSB)},
  pages 4--15, Hawaii, USA, 2005.

\bibitem[Kashima et~al.(2003)Kashima, Tsuda, and Inokuchi]{Kashima:2003}
Hisashi Kashima, Koji Tsuda, and Akihiro Inokuchi.
\newblock Marginalized kernels between labeled graphs.
\newblock In \emph{Proceedings of the 20th International Conference on Machine
  Learning (ICML)}, pages 321--328, Washington, DC, USA, 2003.

\bibitem[Kondor and Borgwardt(2008)]{Kondor:2008}
Risi Kondor and Karsten~M. Borgwardt.
\newblock The skew spectrum of graphs.
\newblock In \emph{Proceedings of the 25th International Conference on Machine
  Learning (ICML)}, pages 496--503, Helsinki, Finland, 2008.

\bibitem[Kondor et~al.(2009)Kondor, Shervashidze, and Borgwardt]{Kondor:2009}
Risi Kondor, Nino Shervashidze, and Karsten~M. Borgwardt.
\newblock The graphlet spectrum.
\newblock In \emph{Proceedings of the 26th International Conference on Machine
  Learning (ICML)}, pages 529--536, Montreal, Quebec, Canada, 2009.

\bibitem[{Kudo} et~al.(2005){Kudo}, {Maeda}, and {Matsumoto}]{Kudo:2005}
Taku {Kudo}, Eisaku {Maeda}, and Yuji {Matsumoto}.
\newblock An application of boosting to graph classification.
\newblock In Lawrence~K. Saul, Yair Weiss, and {L\'{e}on} Bottou, editors,
  \emph{Advances in Neural Information Processing Systems 17}, pages 729--736.
  MIT Press, Cambridge, MA, 2005.

\bibitem[Kuramochi and Karypis(2001)]{Kuramochi:2001}
Michihiro Kuramochi and George Karypis.
\newblock Frequent subgraph discovery.
\newblock In \emph{Proceedings of the 2001 First IEEE International Conference
  on Data Mining (ICDM)}, pages 313--320, San Jose, California, USA, 2001.

\bibitem[Kuramochi and Karypis(2004)]{Kuramochi:2004}
Michihiro Kuramochi and George Karypis.
\newblock An efficient algorithm for discovering frequent subgraphs.
\newblock \emph{IEEE Transactions on Knowledge and Data Engineering},
  16\penalty0 (9):\penalty0 1038--1051, 2004.

\bibitem[Mah\'{e} and Vert(2009)]{Mahe:2009}
Pierre Mah\'{e} and Jean-Philippe Vert.
\newblock Graph kernels based on tree patterns for molecules.
\newblock \emph{Machine Learning}, 75\penalty0 (1):\penalty0 3--35, 2009.

\bibitem[Mah\'{e} et~al.(2005)Mah\'{e}, Ueda, Akutsu, Perret, and
  Vert]{Mahe:2005}
Pierre Mah\'{e}, Nobuhisa Ueda, Tatsuya Akutsu, Jean-Luc Perret, and
  Jean-Philippe Vert.
\newblock Graph kernels for molecular structure-activity relationship analysis
  with support vector machines.
\newblock \emph{Journal of Chemical Information and Modeling}, 45\penalty0
  (4):\penalty0 939--951, 2005.

\bibitem[Mah\'{e} et~al.(2006)Mah\'{e}, Ralaivola, Stoven, and Vert]{Mahe:2006}
Pierre Mah\'{e}, Liva Ralaivola, V\'{e}ronique Stoven, and Jean-Philippe Vert.
\newblock The pharmacophore kernel for virtual screening with support vector
  machines.
\newblock \emph{Journal of Chemical Information and Modeling}, 46\penalty0
  (5):\penalty0 2003--2014, 2006.

\bibitem[Morishita(2002)]{Morishita:2002}
Shinichi Morishita.
\newblock Computing optimal hypotheses efficiently for boosting.
\newblock In Setsuo Arikawa and Ayumi Shinohara, editors, \emph{Progress in
  Discovery Science, Final Report of the Japanese Discovery Science Project},
  volume 2281 of \emph{Lecture Notes in Computer Science}, pages 471--481.
  Springer, 2002.

\bibitem[Nowozin et~al.(2007)Nowozin, Tsuda, Uno, Kudo, and
  Bakır]{Nowozin:2007}
Sebastian Nowozin, Koji Tsuda, Takeaki Uno, Taku Kudo, and G\"{o}khan Bakır.
\newblock Weighted substructure mining for image analysis.
\newblock In \emph{Proceedings of the IEEE Computer Society Conference on
  Computer Vision and Pattern Recognition (CVPR)}, pages 1--8, Minneapolis,
  Minnesota, USA, 2007.

\bibitem[Ralaivola et~al.(2005)Ralaivola, Swamidass, Saigo, and
  Baldi]{Ralaivola:2005}
Liva Ralaivola, Sanjay~J. Swamidass, Hiroto Saigo, and Pierre Baldi.
\newblock Graph kernels for chemical informatics.
\newblock \emph{Neural Networks}, 18\penalty0 (8):\penalty0 1093--1110, 2005.

\bibitem[Rogers and Hahn(2010)]{Rogers:2010}
David Rogers and Mathew Hahn.
\newblock Extended-connectivity fingerprints.
\newblock \emph{Journal of Chemical Information and Modeling}, 50\penalty0
  (5):\penalty0 742--754, 2010.

\bibitem[Saigo and Tsuda(2008)]{Saigo:2008b}
Hiroto Saigo and Koji Tsuda.
\newblock Iterative subgraph mining for principal component analysis.
\newblock In \emph{Proceedings of the 2008 Eighth IEEE International Conference
  on Data Mining (ICDM)}, pages 1007--1012, Pisa, Italy, 2008.

\bibitem[Saigo et~al.(2008)Saigo, Kr\"{a}mer, and Tsuda]{Saigo:2008a}
Hiroto Saigo, Nicole Kr\"{a}mer, and Koji Tsuda.
\newblock Partial least squares regression for graph mining.
\newblock In \emph{Proceeding of the 14th ACM SIGKDD international conference
  on Knowledge discovery and data mining}, pages 578--586, Las Vegas, Nevada,
  USA, 2008.

\bibitem[Saigo et~al.(2009)Saigo, Nowozin, Kadowaki, Kudo, and
  Tsuda]{Saigo:2009}
Hiroto Saigo, Sebastian Nowozin, Tadashi Kadowaki, Taku Kudo, and Koji Tsuda.
\newblock g{B}oost: a mathematical programming approach to graph classification
  and regression.
\newblock \emph{Machine Learning}, 75:\penalty0 69--89, 2009.

\bibitem[Shervashidze et~al.(2011)Shervashidze, Schweitzer, van Leeuwen,
  Mehlhorn, and Borgwardt]{Shrvashidze:2011}
Nino Shervashidze, Pascal Schweitzer, Erik~Jan van Leeuwen, Kurt Mehlhorn, and
  Karsten~M. Borgwardt.
\newblock Weisfeiler-lehman graph kernels.
\newblock \emph{Journal of Machine Learning Research}, 12\penalty0
  (Sep):\penalty0 2539--2561, 2011.

\bibitem[Takigawa and Mamitsuka(2011)]{Takigawa:2011a}
Ichigaku Takigawa and Hiroshi Mamitsuka.
\newblock Efficiently mining {$\delta$}-tolerance closed frequent subgraphs.
\newblock \emph{Machine Learning}, 82\penalty0 (2):\penalty0 95--121, 2011.

\bibitem[Takigawa and Mamitsuka(2013)]{Takigawa:2013}
Ichigaku Takigawa and Hiroshi Mamitsuka.
\newblock Graph mining: procedure, application to drug discovery and recent
  advances.
\newblock \emph{Drug Discovery Today}, 18\penalty0 (1-2):\penalty0 50--57,
  2013.

\bibitem[Tseng and Yun(2009)]{Tseng:2009}
Paul Tseng and Sangwoon Yun.
\newblock A coordinate gradient descent method for nonsmooth separable
  minimization.
\newblock \emph{Mathematical Programming}, 117:\penalty0 387--423, 2009.

\bibitem[Tsuda(2007)]{Tsuda:2007}
Koji Tsuda.
\newblock Entire regularization paths for graph data.
\newblock In \emph{Proceedings of the 24th International Conference on Machine
  learning (ICML)}, pages 919--926, Banff, Alberta, Canada, 2007.

\bibitem[Tsuda and Kudo(2006)]{Tsuda:2006}
Koji Tsuda and Taku Kudo.
\newblock Clustering graphs by weighted substructure mining.
\newblock In \emph{Proceedings of the 23rd International Conference on Machine
  Learning (ICML)}, pages 953--960, Pittsburgh, Pennsylvania, USA, 2006.

\bibitem[Tsuda and Kurihara(2008)]{Tsuda:2008}
Koji Tsuda and Kenichi Kurihara.
\newblock Graph mining with variational dirichlet process mixture models.
\newblock In \emph{Proceedings of the SIAM International Conference on Data
  Mining (SDM)}, pages 432--442, Atlanta, Georgia, USA, 2008.

\bibitem[Vert(2006)]{Vert:2006}
Jean-Philippe Vert.
\newblock Classification of biological sequences with kernel methods.
\newblock In \emph{Proceedings of The 8th International Colloquium on
  Grammatical Inference (ICGI)}, pages 7--18, Tokyo, Japan, 2006.

\bibitem[Vert et~al.(2007)Vert, Qiu, and Noble]{Vert:2007}
Jean-Philippe Vert, Jian Qiu, and William~S. Noble.
\newblock A new pairwise kernel for biological network inference with support
  vector machines.
\newblock \emph{BMC Bioinformatics}, 8\penalty0 (Suppl 10):\penalty0 S8, 2007.

\bibitem[Vishwanathan et~al.(2010)Vishwanathan, Schraudolph, Kondor, and
  Borgwardt]{Vishwanathan:2010}
S.~V.~N. Vishwanathan, Nicol~N. Schraudolph, Risi Kondor, and Karsten~M.
  Borgwardt.
\newblock Graph kernels.
\newblock \emph{Journal of Machine Learning Research}, 11:\penalty0 1201--1242,
  2010.

\bibitem[Wale et~al.(2008)Wale, Watson, and Karypis]{Wale:2008}
Nikil Wale, Ian~A. Watson, and George Karypis.
\newblock Comparison of descriptor spaces for chemical compound retrieval and
  classification.
\newblock \emph{Knowledge and Information Systems}, 14\penalty0 (3):\penalty0
  347--375, 2008.

\bibitem[Warmuth et~al.(2008{\natexlab{a}})Warmuth, Glocer, and
  R\"{a}tsch]{Warmuth:2007}
Manfred Warmuth, Karen Glocer, and Gunnar R\"{a}tsch.
\newblock Boosting algorithms for maximizing the soft margin.
\newblock In J.C. Platt, D.~Koller, Y.~Singer, and S.~Roweis, editors,
  \emph{Advances in Neural Information Processing Systems 20}, pages
  1585--1592. MIT Press, Cambridge, Massachusetts, USA, 2008{\natexlab{a}}.

\bibitem[Warmuth et~al.(2008{\natexlab{b}})Warmuth, Glocer, and
  Vishwanathan]{Warmuth:2008}
Manfred~K. Warmuth, Karen~A. Glocer, and S.~V.~N. Vishwanathan.
\newblock Entropy regularized {LPBoost}.
\newblock In \emph{Proceedings of the 19th International Conference on
  Algorithmic Learning Theory (ALT)}, pages 256--271, Budapest, Hungary,
  2008{\natexlab{b}}.

\bibitem[Wen et~al.(2010)Wen, Yin, Goldfarb, and Zhang]{Wen:2010}
Zaiwen Wen, Wotao Yin, Donald Goldfarb, and Yin Zhang.
\newblock A fast algorithm for sparse reconstruction based on shrinkage,
  subspace optimization, and continuation.
\newblock \emph{SIAM J. Scientific Computing}, 32\penalty0 (4):\penalty0
  1832--1857, 2010.

\bibitem[Wright et~al.(2009)Wright, Nowak, and Figueiredo]{Wright:2009}
Stephen~J. Wright, Robert~D. Nowak, and M\'{a}rio A.~T. Figueiredo.
\newblock Sparse reconstruction by separable approximation.
\newblock \emph{IEEE Transactions on Signal Processing}, 57:\penalty0
  2479--2493, 2009.

\bibitem[Yamanishi et~al.(2007)Yamanishi, Bach, and Vert]{Yamanishi:2007}
Yoshihiro Yamanishi, Francis Bach, and Jean-Philippe Vert.
\newblock Glycan classification with tree kernels.
\newblock \emph{Bioinformatics}, 23\penalty0 (10):\penalty0 1211--1216, 2007.

\bibitem[Yan and Han(2002)]{Yan:2002}
Xifeng Yan and Jiawei Han.
\newblock g{S}pan: Graph-based substructure pattern mining.
\newblock In \emph{Proceedings of the 2002 IEEE International Conference on
  Data Mining (ICDM)}, pages 721--724, Washington, DC, USA, 2002.

\bibitem[Yun and Toh(2011)]{Yun:2011}
Sangwoon Yun and Kim-Chuan Toh.
\newblock A coordinate gradient descent method for $\ell_1$-regularized convex
  minimization.
\newblock \emph{Computational Optimization and Applications}, 48:\penalty0
  273--307, 2011.

\end{thebibliography}

\end{document}